%% file: log_2026.tex
\documentclass{article}
\usepackage[preprint]{log_2026}			

\usepackage{booktabs}						
\usepackage{multirow}						
\usepackage{amsfonts}						
\usepackage{graphicx}						
\usepackage{duckuments}						

\usepackage[utf8]{inputenc} 
\usepackage[T1]{fontenc}    
\usepackage{url}            
\usepackage{nicefrac}       
\usepackage{microtype}      
\usepackage{xcolor}         
\usepackage{microtype}
\usepackage{adjustbox,graphicx}
\usepackage{subcaption}
\usepackage{amsmath,amssymb,amsthm,mathtools}
\newtheorem*{proposition*}{Proposition}
\usepackage{algorithm,algorithmic}

\newcommand\modelname{ProtScape}

\usepackage{nicematrix}
\usepackage{multirow,multicol,booktabs,makecell,nicematrix}

\definecolor{BSUb}{RGB}{0,50,160}     
\definecolor{BSUb2}{RGB}{0,30,96}     
\definecolor{BSUb3}{RGB}{0,114,206}   
\definecolor{BSUo}{RGB}{250,70,22}    
\definecolor{BSUo2}{RGB}{255,130,0}   

\newif\ifrevisions
\revisionstrue
\definecolor{RevColor}{RGB}{0,70,160}
\ifrevisions
  
  \newenvironment{revblock}{\color{RevColor}}{}
\else

\fi

\usepackage[numbers,compress,sort]{natbib}	

\title{\textsc{ProtScape}: A molecular structure and energy-aware representation for protein conformation generation}

\author[Viswanath, Sun et al.]{%
\normalsize
\begin{tabular}{@{}c@{}}
Siddharth Viswanath\textsuperscript{1}\thanks{Equal contribution.} \quad
Xingzhi Sun\textsuperscript{1}\footnotemark[1] \quad
Lucas Lee\textsuperscript{1} \quad
Danqi Liao\textsuperscript{1} \\
Hiren Madhu\textsuperscript{1} \quad
David R.~Johnson\textsuperscript{2} \quad
Jo\~{a}o Felipe Rocha\textsuperscript{1} \quad
Egbert Castro\textsuperscript{1} \\
Jackson D.~Grady\textsuperscript{1} \quad
Michael Perlmutter\textsuperscript{2} \quad
Dhananjay Bhaskar\textsuperscript{3} \quad
Smita Krishnaswamy\textsuperscript{1}
\end{tabular} \\[0.5em]
\textsuperscript{1}Yale University \quad
\textsuperscript{2}Boise State University \quad
\textsuperscript{3}University of Wisconsin--Madison
}

\begin{document}

\maketitle

\begin{abstract}
Molecular dynamics (MD) simulations are a principled but computationally expensive approach for studying protein conformational variability, making it challenging to generate large ensembles of structures or to characterize transitions between metastable conformations. AI methods for upsampling MD simulations have been developed in recent years, but struggle due to difficulties of sampling complex, high-dimensional molecular distributions. One avenue that these methods overlook is to learn a latent space in which such sampling becomes easier. In order to address this, we introduce ProtScape, a generative geometric deep learning framework that learns structure and energy-aware representations of protein conformational landscapes from MD simulations. ProtScape represents protein conformations using an equivariant graph neural network and a multiscale, deep wavelet transform that captures both local geometric interactions and nonlocal, collective motions. This latent representation of the model is dual-organized: 1) by the structure captured by wavelet transform layers, via a reconstruction error; 2) by energy using a Laplacian energy smoothness penalty. This structured latent space enables meaningful generation and exploration of protein conformations. ProtScape supports multiple generative modes including 1) \textit{ensemble generation,} which upsamples ensembles given limited MD trajectories flow matching from noise to the organized manifold of conformations, 2) minimum-energy path generation between two high-energy conformations, which is guided by energy using a nudged elastic band formulation, and 3) energy-descent, which generates trajectories towards lower energies from a high-energy conformation using gradient descent, thereby hypothesizing folding or other stabilizing trajectories. Across five proteins from the ATLAS molecular dynamics dataset, ProtScape generates ensembles that are closer in distribution to the reference simulation than the strongest baseline on every protein, and transition paths in which under $6\%$ of residues occupy disallowed backbone conformations, against over $40\%$ for that baseline. Starting from a high-energy structure, its energy-descent trajectories end below their starting energy in $94\%$ of runs. Overall, organizing a latent space by both structure and energy lets one representation support ensemble sampling, minimum-energy path finding and energy-guided descent on the same conformational landscape.
\end{abstract}

\section{Introduction}

Protein function does not arise from a single static structure, but from a collection of conformations and the ability to move between them~\cite{eisenmesser2005intrinsic,wolf2004linkage}.
Molecular dynamics (MD) simulations provide a physically grounded way to model these motions by integrating Newton's equations on an atomistic energy landscape~\cite{karplus2005molecular}. In principle, MD captures both conformational variability and transitions between states. In practice, however, MD simulations are severely limited by computational cost. Simulations typically access microsecond timescales, while many biologically relevant conformational changes occur on millisecond or longer timescales~\cite{lisi2017allostery}. As a result, MD often undersamples the conformational landscape and rarely observes transitions between metastable states without extensive simulation time.

These limitations motivate two fundamental challenges in protein modeling. First, how can we efficiently generate diverse protein conformations that reflect the broader conformational landscape, including regions that are sparsely sampled by MD? Second, how can we interpolate between two conformational states, learning how a protein can transition from one state to another, without performing prohibitively expensive MD simulations and waiting for the transition to occur? Both problems are challenging to address directly in ambient coordinate space due to the large number of degrees of freedom and the strong geometric and physical constraints imposed by molecular structure. This motivates the development of methods that explicitly learn the manifold structure of protein conformational landscapes, revealing the space of physically viable conformations rather than operating purely in ambient coordinate space. 

Recent deep learning approaches have begun to address these challenges but face two persistent limitations. First, their underlying representations are typically based on local message passing or sequence tokens, which fail to capture the multiscale geometric structure of proteins ~\cite{jing2024alphafold, jing2024generative}. As a result, these models often struggle to represent the manifold of physically viable conformations and the collective motions underlying conformational change. Second, many approaches operate as data-driven generators without explicit alignment to the underlying physical energy landscape~\cite{jing2024alphafold, jing2024generative, daigavane2025jamun}. Together, these limitations make it difficult to scale beyond short peptides, to interpolate between distant conformational states, or to generate trajectories that remain physically consistent with the underlying energy landscape. 

In this work, we introduce \textbf{\modelname{}}\footnote{Code is available here: https://github.com/KrishnaswamyLab/ProtSCAPE-Net}, a geometric deep learning framework that learns a structured, energy-regularized latent representation of protein conformational landscapes from MD trajectories. We make the following contributions.

\begin{itemize}
\setlength{\itemsep}{2pt}\setlength{\parskip}{0pt}\setlength{\topsep}{4pt}
\item \textbf{A structure- and energy-aware representation of conformational landscapes.} \modelname{} combines an SE(3)-equivariant graph neural network (GNN) with deep graph wavelet transforms to capture both local geometric features and nonlocal collective motions, while an energy-based smoothness penalty aligns the latent space with the underlying potential energy landscape.

\item \textbf{Three generative modes on a single latent space.} The same representation supports latent diffusion for ensemble upsampling, latent energy-descent for generating relaxation trajectories toward lower-energy states, and a latent-space Nudged Elastic Band (NEB) method for minimum-energy path estimation between metastable conformations.

\item \textbf{Protein-specific representation learning.} Unlike methods that train across large trajectory collections, \modelname{} learns a representation per protein, which organizes the conformational landscape explicitly by structure and energy and thereby supports energy-guided descent and minimum-energy path generation that are not feasible in more generalized settings.

\item \textbf{Evaluation on real protein systems.} We demonstrate \modelname{} on proteins from the ATLAS dataset~\cite{10.1093/nar/gkad1084}, showing that it scales beyond short peptides and supports both conformational generation and interpolation.
\end{itemize}

\section{Background}
\label{sec:bg}
In this section, we give relevant background on the geometric scattering transform and equivariant GNNs. For a discussion of related work on molecular dynamics, please see Appendix \ref{app: related}.

Throughout this section, we let 
 $G=(V,E)$ denote an undirected, connected graph, with vertices $V=\{v_1,\ldots,v_N\}$ and edge set $E$. We let $A$ denote its adjacency matrix and let $D$ denote the corresponding diagonal degree matrix, i.e., $D=\text{diag}(\mathbf{d}), \mathbf{d}=A\mathbf{1}$. We will assume that we are given a collection of node features organized into a matrix $X$, where each row of this matrix $\mathbf{x}_v$ contains all of the features associated with node $v$. Each column of $X$, $\mathbf{x}^{(f)}$ contains the values of the $f$-th feature at all of the vertices and therefore may be thought of as a signal (function) defined on $V$ (where the $i$-th entry of $\mathbf{x}^{(f)}$ is the value of this signal at $v_i$, i.e., $\mathbf{x}^{(f)}(v_i)$).

\subsection{Diffusion Wavelets and the Geometric Scattering Transform}\label{sec: background - wavelets and scattering}

The geometric scattering transform \citep{gao2019geometric,gama2019diffusion,zou_graph_2019} is a multi-layer, feed-forward feature extractor for processing the node features. 
A key component of this scattering transform is Diffusion Wavelets \citep{coifman:diffWavelets2006}, 
which are defined in terms of the diffusion matrix $P=\frac{1}{2}\left(I+AD^{-1}\right)$ that describes the transition probabilities of a lazy random walk.

Given $P$,  the diffusion wavelets $\mathcal{W}_J=\{\Psi_j\}_{j=0}^J$, 
 are defined by $\Psi_0= I-P$, and
%
$$\Psi_j 
= P^{2^{j-1}}-P^{2^j},\text{ for }1\leq j \leq J.$$ 
%
To understand these wavelets, we let $\mathbf{x}\in\mathbb{R}^N$ be a graph signal (suppressing the superscript $f$) and note that the multiplication $\mathbf{x}\rightarrow P\mathbf{x}$ effectively averages each entry $\mathbf{x}(v_i)$ with the values of $\mathbf{x}$ at neighboring vertices. Therefore, for $t\geq 1$ we may interpret $P^t\mathbf{x}$ as the output of a simple $t$-layer message passing neural network with a receptive field of length $t$ (omitting the nonlinearities and feature-combination steps). 

For molecular graphs, it is important to capture global geometry, which requires large values of $t$. However, $P^t$ is known to act as a low-pass filter (i.e., a smoothing operator that depresses the oscillatory parts of a signal), and therefore representing the signal $\mathbf{x}$ by $P^t\mathbf{x}$ with a large value of $t$ would suffer from the over-smoothing problem (see, e.g., \cite{rusch2023survey}).
However, since the wavelets consider \emph{differences} between diffusion scales, i.e., $P^{2^{j-1}}-P^{2^{j}}$, they do not suffer from over smoothing \citep{wenkel2022overcoming}. 

Given the wavelet transform,  scattering coefficients are obtained by iteratively applying wavelet multiplications with applications of the entrywise modulus (absolute value) operator $M\mathbf{x}(v)=|\mathbf{x}(v)|$. Namely, the first- and second-order scattering coefficients of $\mathbf{x}$ at vertex $v$ are given by 
\begin{equation*}
    U[j]\mathbf{x}(v)= M\Psi_j\mathbf{x}(v)=|\Psi_j\mathbf{x}(v)|, \quad\text{and}\quad U[j_1,j_2]\mathbf{x}(v)= M\Psi_{j_2}M\Psi_{j_1}\mathbf{x}(v)=|\Psi_{j_2}|\Psi_{j_1}\mathbf{x}(v)||,
\end{equation*}
If desired, higher-order coefficients can be defined in a similar manner, e.g.,
$U[j_1,j_2,j_3]\mathbf{x}(v)= M\Psi_{j_3}M\Psi_{j_2}M\Psi_{j_1}\mathbf{x}(v)$.  
Notably, the geometric scattering transform (and its variations) have been shown empirically to be particularly useful for molecular graphs \cite{gao2019geometric,9943379}, where global geometry is important, and also have strong theoretical guarantees ensuring stability to small perturbations of the graph structure \cite{gama2019diffusion,zou_graph_2019,perlmutter2023understanding}.

\subsection{Equivariant GNNs}\label{sec: background EGNN}

When working with molecular graphs, it is natural to consider the vertices (i.e., atoms) as lying in $3D$ space. We typically are given scalar-valued node features such as the atomic numbers, and vector-valued node features such as atomic coordinates. If one rotates a molecule by some rotation $R\in SO(3)$, the scalar-valued features will not change. This is a property referred to as \emph{rotational invariance}. On the other hand, the vector-valued node features will be rotated in the same manner. For instance, the coordinates of each vertex $\mathbf{c}_v\in\mathbb{R}^{1\times 3}$ will be transformed via $\mathbf{c}_v \leftarrow \mathbf{c}_vR$. This property is referred to as \emph{rotational equivariance}. Equivariant GNNs \citep{satorras2022enequivariantgraphneural,batzner20223,han2022geometrically,duval2023hitchhiker} are designed to process both of the scalar-valued and vector-valued features in a manner that respects these intrinsic symmetries (with analogous properties for edge features such as bond lengths or angles). In our architecture, we will utilize the equivariant GNN (EGNN) introduced in \cite{satorras2022enequivariantgraphneural}.

\section{Methods}
\begin{figure*}[!t]
\centering
\includegraphics[width=0.99\textwidth]{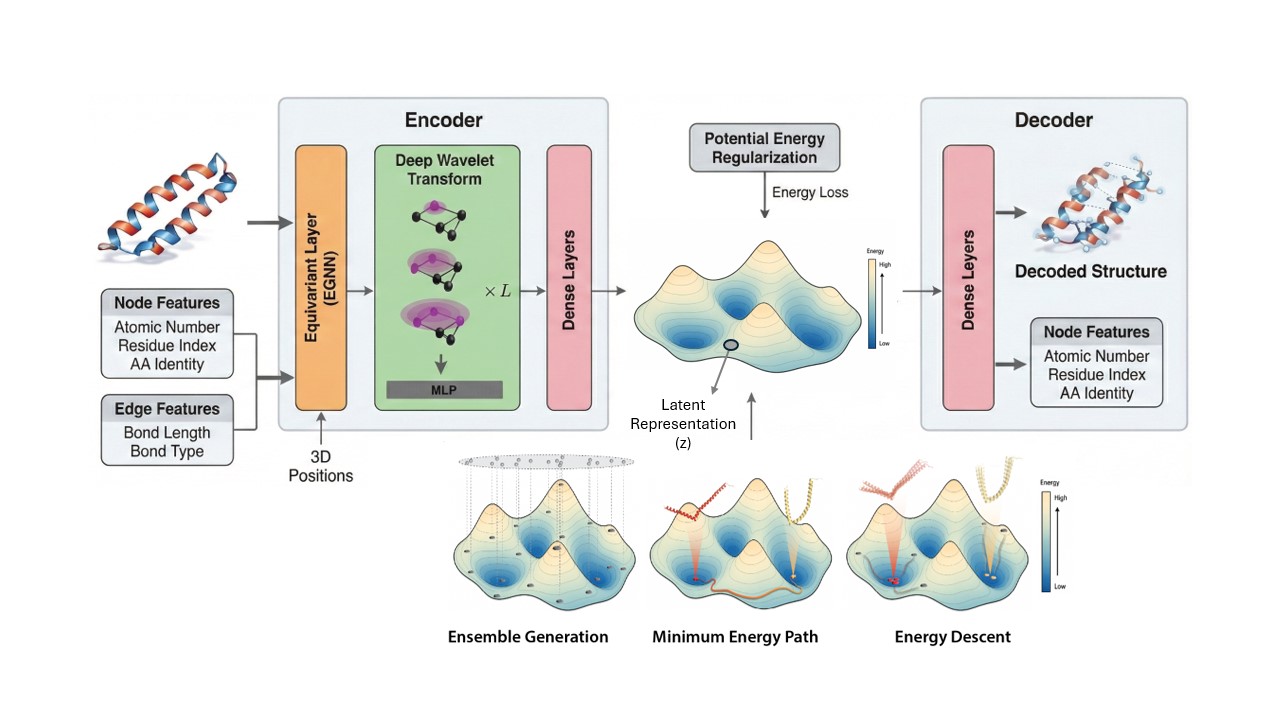}
\caption{\textbf{Overview of the \modelname{} framework.}
The model encodes atomistic protein features into a low-dimensional latent manifold using an SE(3)-equivariant network with deep graph wavelet transforms. The latent space is regularized by potential energy to ensure physical validity. This representation supports three complementary tasks: \textit{Ensemble Generation}, utilizing latent diffusion to sample diverse equilibrium states, \textit{Minimum Energy Path Interpolation}, employing Latent NEB to optimize minimum energy paths between conformations, and \textit{Energy-Guided Generation}, performing energy-guided latent dynamics to iteratively refine conformations toward lower-energy and structurally stable states.
\vspace{-.2cm}
} 
\label{fig:schematic}
\end{figure*}

\modelname{} has three components (Figure~\ref{fig:schematic}): a graph-based representation of protein conformations, an autoencoder built from an SE(3)-equivariant GNN and a deep graph wavelet transform that learns an energy-regularized latent space, and a suite of generative procedures that operate on this latent space to support ensemble generation and energy-guided exploration of the conformational landscape.

\subsection{Graph Construction and Node Features}
\label{sec:graph}

Each MD simulation frame is represented as an atom-level graph, together with node features and edge features encoding the protein's covalent structure, geometry, and physical state.  Specifically, we construct an undirected graph $G = (V,E)$, where each node $v\in V$ corresponds to an atom and each edge $(i,j)$ represents the covalent bond between atoms $i$ and $j$.

We will let $\mathbf{c}^{(t)}_v$ denote the Cartesian coordinates of $v_i$ at time $t$, and will assume that we  have additional node features which are given by
%
%
$ 
    \mathbf{x}_v
    = [Z_v||r_v||a_v],
$ 
where $Z_v$ denotes a one-hot encoding of the atomic number of atom $v$,
$r_v$ is the index of the residue containing the atom, and $a_v$ is the amino-acid index of that residue. Proteins are composed of chains of amino acids, and each amino acid unit within the chain is referred to as a \textit{residue}. The residue index $r_v$ thus encodes the position of the atom along the protein backbone, while the amino-acid index $a_v$ encodes the specific amino acid type at that position (e.g., alanine, glycine, or lysine). These features encode chemical and sequence-level information and are constant over time. For each bonded atom pair, $(v_i,v_j)$, at each time $t$, we construct an edge attribute vector:
$
    e_{ij}^{(t)} = [l^{eq}_{ij}||l^{(t)}_{ij}||\Delta l^{(t)}_{ij}||b_{ij}] 
$ 
where $l^{eq}_{ij}$ is the equilibrium bond length computed from the first frame of the trajectory, $l^{(t)}_{ij}$ is the instantaneous bond length at time $t$, and $\Delta l^{(t)}_{ij} = l^{(t)}_{ij} - l^{eq}_{ij}$ is the bond length deviation. The vector $b_{ij}$ is a one-hot encoding of the bond type (single, double, triple, or aromatic). 


\subsection{Deep Wavelet Transform}
\label{sec:wavelet}

A key piece of our architecture (which will be fully described in the following sections) is a generalization of the geometric scattering transform (see Section \ref{sec: background - wavelets and scattering}) which we name the \emph{Deep Wavelet Transform}. The utilization of this novel wavelet transform is motivated by the fact that 
protein conformational changes are inherently multiscale, involving both local atomic rearrangements and coordinated long-range collective motions. As discussed in Section~\ref{sec: background - wavelets and scattering}, the geometric scattering transform provides a principled framework for capturing multiscale graph structure through diffusion wavelets while avoiding the oversmoothing behavior of repeated message passing. Building on this framework, \modelname{} incorporates a learnable graph wavelet transform for modeling protein conformational landscapes. 

In each layer $l$, for $j=0$ to $J$, we  define a generalized $j$-hop diffusion aggregation operator as a weighted combination of the diffusion powers:
\begin{equation*}
    \Phi^{(l)}_jX = \sum_{k=0}^j F[l,k]P^kX.
\end{equation*}
where $X$ is our initial node-feature matrix and $F \in \mathbb{R}^{L\times j}$ is a learnable weight matrix. Then, analogous to Section \ref{sec: background - wavelets and scattering}, we define $\smash{\Psi_j^{(l)} = \Phi_{j-1}^{(l)} - \Phi_{j}^{(l)}}$ (although here, for increased descriptive power, we take the diffusion scales to be consecutive integers rather than powers of $2$ so, e.g., $\Psi_3^{(l)}=\Phi_2^{(l)}-\Phi_3^{(l)}$. We then stack  of the $\Psi_j^{(l)}$ horizontally to form a matrix $\Psi^{(l)}\in\mathbb{R}^{(J+1)N\times N}$.

%
%
For $l=1$ to $L$, the  Deep Wavelet Transform uses a layerwise update rule built from a multilayer perceptron ($\textbf{MLP}$):
 \begin{equation*}
     H^{(l)} = \textbf{MLP}(\Psi^{(l)}X),
 \end{equation*}
 (with $H^{(0)}=X$). 
In each layer, the wavelet transform $\Psi^{(l)}$ mixes information across nodes and $\textbf{MLP}$ mixes information across channels (features). Stacking multiple such layers increases the receptive field while maintaining control over scales at which the messages are aggregated. We denote the output of the Deep Wavelet Transform by $\textbf{DWT}(X)=H^{(L)}$. 

 

\subsection{Learning Latent Representations for Protein Conformations}\label{sec:train}

ProtScape is trained as an autoencoder framework that learns latent representations of protein conformations jointly organized by molecular structure and physical energy. Given an atom-level protein graph with node features and Cartesian coordinates as described in Section~\ref{sec:graph}, the encoder maps the conformation into a latent representation $z^{(t)}$, while the decoder reconstructs both the molecular geometry and node features. In parallel, we regularize the latent manifold using molecular energy computations so that nearby latent representations correspond to energetically similar conformations. 

\textbf{Encoder: } The encoder is a composition of (i) EGNN \cite{satorras2022enequivariantgraphneural}, an SE(3) equivariant geometric network, followed by (ii) the deep wavelet transform as described in Section~\ref{sec:wavelet}. We let $\mathbf{x}_v = [Z_v||r_v||a_v]$ be the input node features and let  $\mathbf{e}_{ij}^{(t)} = [l_{ij}^{eq}||l_{ij}^{(t)}||\Delta l_{ij}^{(t)}||b_{ij}]$ be the input edge features at time $t$ defined in Section~\ref{sec:graph}. (Note that the edge features include quantities such as bond length which vary over time, whereas the node features $\mathbf{x}_v$ include quantities such as the atomic number which do not.) We let the $\mathbf{c}_v^{(t)}$ be the $3D$ Cartesian coordinates (nm) of the atom at time $t$ (written as a row vector). We first apply a stack of EGNN~\cite{satorras2022enequivariantgraphneural} layers that updates the node features and the coordinates: $\tilde{\mathbf{x}}_v^{(t)}=[\mathbf{x}_v||\mathbf{c}_v^{(t)}] = \textbf{EGNN}(\mathbf{x}_v,\mathbf{c}_v^{(t)},\mathbf{e}_{ij}^{(t)})$. This module injects a geometry-respecting inductive bias, making the model insensitive to global rotation and translations while remaining sensitive to spatial structure of the protein. 


We apply the deep wavelet transform  defined in Section~\ref{sec:wavelet} to extract multiscale representations $ W^{(t)} = \textbf{DWT}(\tilde{X}^{(t)}),$ where $\tilde{X}^{(t)}$ is the transformed feature matrix with rows $\tilde{\mathbf{x}}_v^{(t)}$.
    

Finally, the output of our encoder is a graph-level representation obtained by summing across nodes, followed by an MLP producing the latent representations $\mathbf{z}^{(t)}$ taking values in latent space $\mathcal{Z}$:
\begin{equation*}
    \mathbf{z}^{(t)} = \textbf{MLP}(\mathbf{g}^{(t)}) \in \mathbb{R}^d, \quad \mathbf{g}^{(t)} = \sum_{v\in V}\mathbf{w}_v^{(t)},
\end{equation*}
where $\mathbf{w}_v^{(t)}$ is the row of $W^{(t)}$ corresponding to $v$.
%

\noindent\textbf{Decoder: } From the latent representations $\mathbf{z}^{(t)}$, the decoder reconstructs both the coordinates and the node features. Importantly, the decoder is not constructed as an inverse wavelet transform; rather, it learns a nonlinear mapping from the latent manifold back to atomistic protein structure. Cartesian coordinates are reconstructed using a regression decoder $\hat{\mathbf{c}}^{(t)} = f_c(\mathbf{z}^{(t)})$ while classification head reconstructs the discrete node features $\hat{\mathbf{x}}^{(t)} = f_x(\mathbf{z}^{(t)})$. We will write the full decoder as $g_{\text{dec}}(\mathbf{z}^{(t)})=[\hat{\mathbf{c}}^{(t)}||\hat{\mathbf{x}}^{(t)}]=[f_c(\mathbf{z}^{(t)})||f_x(\mathbf{z}^{(t)})].$

\noindent\textbf{Reconstruction Objective: } The model is trained to reconstruct both coordinates and node-level features. Coordinate reconstruction is supervised using an aligned mean squared error loss: $\mathcal{L}_c = \text{MSE}(\text{Align}(\hat{\mathbf{c}}^{(t)}, \mathbf{c}^{(t)}),\hat{\mathbf{c}}^{(t)})$ where Align() denotes optimal rigid-body alignment between predicted and target coordinates. Complete details of the alignment procedure are provided in the Appendix~\ref{app:implementation}. The node features $\mathbf{x}$ are reconstructed using a negative log-likelihood objective with softmax over atomic number, residue index, and amino acid identity: $\mathcal{L}_x = \mathcal{L}_Z + \mathcal{L}_r + \mathcal{L}_a$.

\noindent\textbf{Energy Regularization: } While the reconstruction objective encourages faithful encoding of molecular structure, it does not ensure that the latent manifold reflects the underlying physical landscape governing conformational dynamics. We therefore introduce an energy-based regularization that reorganizes the latent space according to potential energy. For each protein conformation, we compute the potential energy $E_i$ using OpenMM with the AMBER force field~\cite{maier2015ff14sb}. Given a batch of latent embeddings $\{\mathbf{z}_i\}^B_{i=1}$ and corresponding energies $\{E_i\}^B_{i=1}$, we construct a weighted affinity graph in latent space with edge weights defined by pairwise similarities: $w_{ij} = \exp\left(-\frac{\left\|\mathbf{z}_i - \mathbf{z}_j\right\|^2}{2\sigma^2}\right)$. In practice, the bandwidth parameter $\sigma$ can be chosen adaptively based on the distance to the $k$-th nearest neighbor or via entropy-based heuristics to account for varying local densities. We then impose a  smoothness penalty on the energy values based on the graph Laplacian $L$ of the weighted affinity graph:
\begin{equation*}
    \mathcal{L}_e =\vec{E}^TL\vec{E}= \frac{1}{2}\sum_{i\neq j}w_{ij}(E_i - E_j)^2.
\end{equation*}
This loss encourages nearby latent representations to correspond to conformations with similar energies, producing a smoother and more physically organized conformational manifold.

The complete training objective is:
\begin{equation*}
    \mathcal{L} = \mathcal{L}_x + \lambda_c\mathcal{L}_c + \lambda_e\mathcal{L}_e.
\end{equation*}
Together, these objectives enable \modelname{} to learn wavelet-based latent representations that capture the multiscale geometric structure of protein conformational landscapes while remaining smoothly organized by potential energy, supporting meaningful conformational generation and interpolation.
%


%

\subsection{Latent Generative Dynamics}
\label{sec:latent_dynamics}

The structured latent space learned in Section~\ref{sec:train} enables generative modeling directly in the latent space $\mathcal{Z}$, sidestepping the cost of atomistic simulation while remaining grounded in the underlying physics. We exploit this in three complementary modes, each addressing a distinct gap in standard MD workflows. Full algorithms, derivations, and implementation details are provided in Appendix~\ref{app:generation}.

\noindent\textbf{Ensemble generation in latent space.}
We train a generative model on the \modelname{} latent representation to sample the equilibrium ensemble. We will consider versions of our model which use either a Denoising Diffusion Probabilistic Model (DDPM)~\citep{ho2020denoising} (see Appendix~\ref{app:generation}) or flow matching~\citep{lipman2022flow}. Both of these choices recover the latent data distribution at convergence; because the latent space is energy-aware, the learned density approximates the projected Boltzmann distribution at the MD simulation temperature, enabling rapid generation of diverse, physically valid conformations that upsample beyond the density of the training trajectory, recovering ensemble statistics that would otherwise require substantially longer MD runs.


\noindent\textbf{Energy-guided generative dynamics.}
Equilibrium sampling alone cannot answer many biophysical questions: docking needs a single low-energy representative of the relevant basin; structure refinement also needs the relaxation path by which a perturbed conformation descends into it; transition mechanism analysis needs the minimum energy path connecting two metastable states. These are deterministic structures on the energy landscape rather than samples from its equilibrium distribution, and conventional MD addresses them through separate enhanced-sampling protocols, each requiring its own re-run. We instead obtain all three from a single trained representation by pulling the AMBER14~\citep{maier2015ff14sb} force back through the decoder Jacobian, $\nabla_{\mathbf{z}} U(\mathbf{z}) = J_{g_\text{dec}}^\top(\mathbf{z}) \cdot \nabla_{\mathbf{x}} U(g_\text{dec}(\mathbf{z})),$
exposing physical forces directly on the latent manifold without any learned surrogate energy.

\textit{Mode 1: Latent NEB.} Given two endpoint conformations, we recover a minimum energy path as a chain of latent pivots updated by the gradient component perpendicular to the path tangent. Crucially, the path is optimized in a low-dimensional latent space where the decoder enforces geometric validity, rather than in $\mathbb{R}^{3N}$ where every pivot must independently satisfy bonded and steric constraints. We further replace standard NEB's spring forces -- notoriously sensitive to spring stiffness -- with iterative arc-length redistribution, removing a tunable hyperparameter and yielding transition paths between arbitrary endpoints without ever observing the transition during MD simulation. Additional details are provided in Appendix~\ref{app:generation}.

\textit{Mode 2: Latent energy descent.} Starting from a high-energy conformation, we evolve $\mathbf{z}_t$ under gradient descent on $U(g_\text{dec}(\mathbf{z}))$, with a learned denoising-autoencoder projector~\citep{vincent2008extracting,alain2014regularized} keeping trajectories on the manifold. This setting models relaxation dynamics on the conformational landscape, where proteins evolve toward energetically favorable conformations. The descent settles into the local basin and recovers its dominant low-energy structure, while the retained intermediate codes trace a relaxation path from the high-energy start to this minimum. A single run produces both the basin minimum and the full descent trajectory leading to it, without re-running MD. Additional details are provided in Appendix~\ref{app:generation}.

Complete implementation details of \modelname{} are provided in Appendix~\ref{app:implementation}.

\section{Empirical results}\label{sec: results}

In this section, we evaluate \modelname{} on molecular dynamics trajectories from the ATLAS dataset~\cite{10.1093/nar/gkad1084}, first analyzing the structure and organization of the learned latent manifold, and subsequently evaluating three downstream generative modeling tasks: (i) ensemble generation, (ii) energy-descent, and (iii) minimum energy path interpolation using NEB. We focus on proteins with substantial conformational variability, selecting trajectories where the maximum pairwise C$\alpha$ root mean square deviation (RMSD) exceeds $4\,\text{\AA}$, ensuring that the proteins exhibit significant conformational shifts. Here, we present results using Mode 1 (latent NEB). Results obtained via Mode 2 (latent energy descent) are presented in Appendix \ref{app:energy_descent}. For the data splits and the baseline training protocol, see Appendix~\ref{app:data}.

\begin{figure}
\centering
\includegraphics[width=0.85\linewidth]{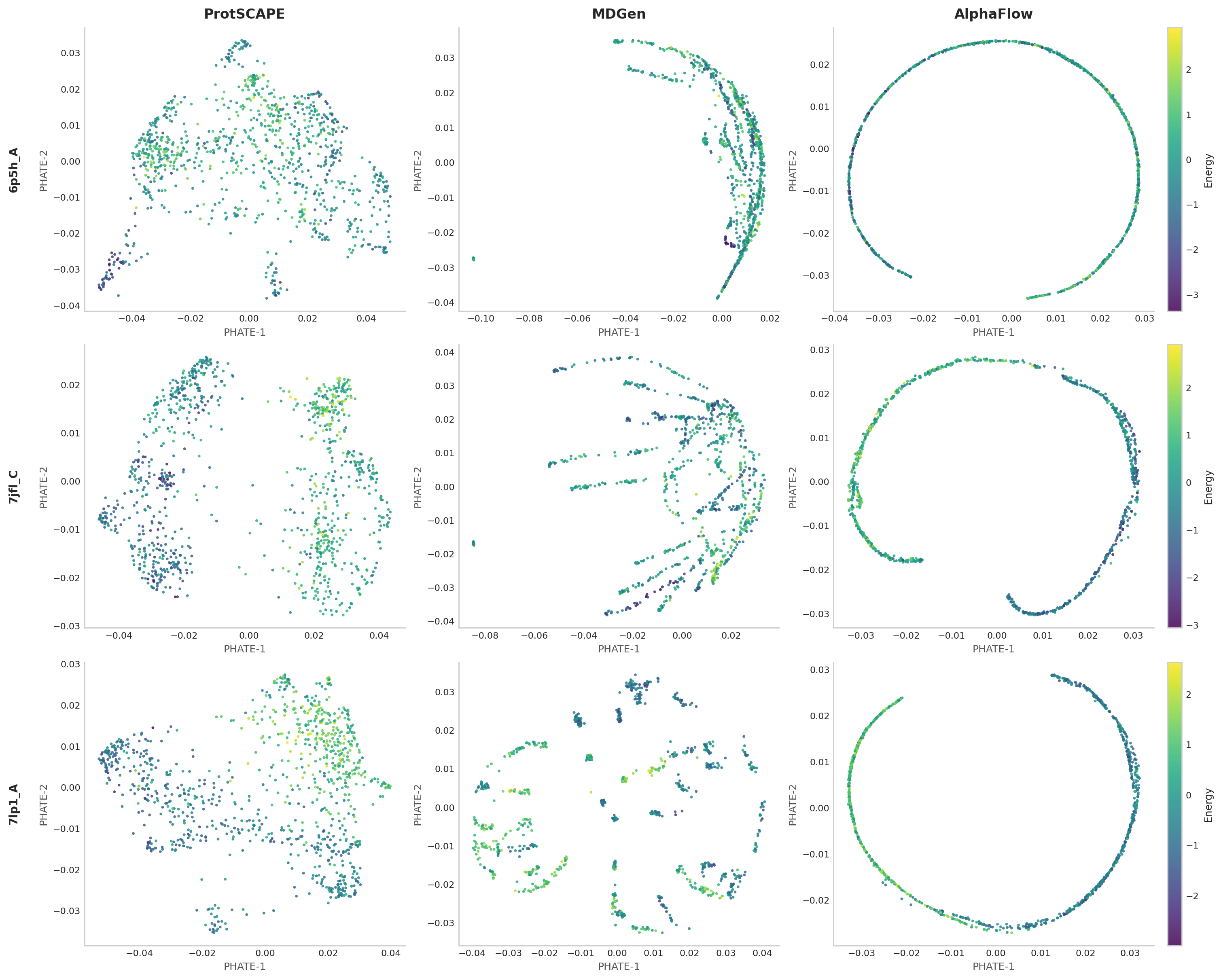}
\caption{\textbf{PHATE embeddings of conformational ensembles colored by potential energy.} Rows correspond to three proteins (6p5hA, 7jflC, 7lp1A); columns correspond to ProtScape, MDGen, and AlphaFlow. Points are colored by potential energy. ProtScape creates a continuous landscape organized by energy with multi-modal distributions reflecting conformation variation; MDGen yields branched, discontinuous structure; AlphaFlow collapses to 1-dimensional trajectories.}
\label{fig:representation}
\end{figure}
\vspace{-6pt}
\noindent\textbf{Metrics: } Generated conformations are evaluated using both distributional and structural quality metrics. For ensemble generation, we measure root mean square fluctuation (RMSF) agreement, pairwise RMSD statistics, transition and variance statistics, and Wasserstein distances in PCA space. Structural validity is evaluated using Ramachandran outliers, clash scores, rotamer outliers, and MolProbity scores. 
For interpolation experiments, we additionally compute RMSD between consecutive decoded conformations to evaluate transition smoothness.

\noindent\textbf{Representation Quality: } We first evaluate the structure of the learned conformational representations. Figure~\ref{fig:representation} visualizes the learned representations colored by potential energy for \modelname{}, MDGen~\cite{jing_generative_2024}, and AlphaFlow~\cite{volk_alphaflow_2023} using the PHATE~\cite{Moon2019} dimensionality reduction algorithm. As we note \modelname{} creates smoother, more continuous distributions in the low-dimensional latent space, also organized by modes of conformation, which is more amenable to sampling diverse ensembles and paths that were not present in the training data. By contrast, MDGen has a more "shattered" collection of trajectories and a discontinuous space where it is harder to learn transitions. On the other hand, Alphaflow collapses all conformations consistently to a 1-dimensional trajectory, thereby losing structural information. ProtScape's organization, which is both structure- and energy-aware, is important for downstream generative modeling, as it enables the traversal of the conformational landscape and supports physically meaningful interpolation and sampling.

\noindent\textbf{Ensemble Generation: } We next evaluate the ability of \modelname{} to generate conformational ensembles that reproduce the underlying MD distribution. Table~\ref{tab:distribution-quality-fm} reports ensemble quality metrics across multiple protein systems. Across nearly all proteins and metrics, \modelname{} achieves substantially lower errors than AlphaFlow and MDGen, indicating closer agreement with the reference conformational distributions. To further evaluate the quality of generated ensembles, we compute residue-level RMSF correlations, pairwise RMSD statistics, and structural validity metrics including Ramachandran outliers, clash scores, rotamer outliers, and MolProbity scores. As shown in the Appendix~\ref{app:generation}, in Tables~\ref{tab:dynamics-quality-fm} and ~\ref{tab:structural-quality-fm}, \modelname{} achieves near-perfect RMSF correlations across proteins while reducing pairwise RMSD discrepancies and maintaining favorable stereochemical quality relative to competing approaches. Ablation studies of the ensemble generation are provided in Appendix~\ref{app:abl}.

\begin{table*}[!htbp]
\caption{\textbf{Distributional ensemble quality metrics for proteins in our dataset.} Lower is better for all reported metrics. Best result in each category is bolded. Values are reported as mean $\pm$ std.}
\label{tab:distribution-quality-fm}
\centering
\setlength{\tabcolsep}{5pt}
\begin{small}
\begingroup
\renewcommand{\arraystretch}{1.2}
\resizebox{0.90\textwidth}{!}{
\begin{NiceTabular}{llccccc}
\toprule
\textbf{Protein} & \textbf{Method}
& RMWD$\downarrow$
& Trans.$\downarrow$
& Var.$\downarrow$
& Pool PCA-W$_2\downarrow$
& Ref PCA-W$_2\downarrow$ \\
\midrule
\multirow{4}{*}{\makecell{6H86\\(Chain A)}}
& AlphaFlow & 4.2487 $\pm$ 0.017 & 2.4578 $\pm$ 0.005 & 3.4657 $\pm$ 0.021 & 34.9194 $\pm$ 0.195 & 34.6883 $\pm$ 0.195 \\
& MDGen & 1.0066 $\pm$ 0.008 & 0.8097 $\pm$ 0.008 & 0.5978 $\pm$ 0.014 & 3.8805 $\pm$ 0.209 & 3.7751 $\pm$ 0.207 \\
& \Block[fill=gray!15]{1-}{}\textbf{ProtScape (DDPM)} & \textbf{0.3441 $\pm$ 0.016} & \textbf{0.2070 $\pm$ 0.013} & \textbf{0.2749 $\pm$ 0.010} & \textbf{1.4490 $\pm$ 0.215} & \textbf{1.4576 $\pm$ 0.215} \\
& \Block[fill=gray!15]{1-}{}\textbf{ProtScape (Flow Matching)} & 0.4244 $\pm$ 0.104 & 0.2555 $\pm$ 0.063 & 0.3386 $\pm$ 0.083 & 1.8907 $\pm$ 1.473 & 1.8973 $\pm$ 1.475 \\
\midrule
\multirow{4}{*}{\makecell{6P5H\\(Chain A)}}
& AlphaFlow & 1.0504 $\pm$ 0.061 & 0.7488 $\pm$ 0.077 & 0.7350 $\pm$ 0.010 & 8.0587 $\pm$ 0.839 & 8.0806 $\pm$ 0.830 \\
& MDGen & 0.6200 $\pm$ 0.009 & 0.4258 $\pm$ 0.013 & 0.4504 $\pm$ 0.015 & 2.2015 $\pm$ 0.131 & 2.0740 $\pm$ 0.140 \\
& \Block[fill=gray!15]{1-}{}\textbf{ProtScape (DDPM)} & \textbf{0.2127 $\pm$ 0.005} & \textbf{0.1424 $\pm$ 0.006} & \textbf{0.1579 $\pm$ 0.002} & \textbf{0.9370 $\pm$ 0.125} & \textbf{0.9449 $\pm$ 0.125} \\
& \Block[fill=gray!15]{1-}{}\textbf{ProtScape (Flow Matching)} & 0.2597 $\pm$ 0.014 & 0.1563 $\pm$ 0.015 & 0.2069 $\pm$ 0.013 & 1.0411 $\pm$ 0.256 & 1.0460 $\pm$ 0.256 \\
\midrule
\multirow{4}{*}{\makecell{6JV8\\(Chain A)}}
& AlphaFlow & 1.1663 $\pm$ 0.012 & 0.6576 $\pm$ 0.015 & 0.9631 $\pm$ 0.007 & 7.8291 $\pm$ 0.179 & 7.8324 $\pm$ 0.175 \\
& MDGen & 0.4729 $\pm$ 0.014 & 0.3818 $\pm$ 0.017 & 0.2787 $\pm$ 0.011 & 2.4517 $\pm$ 0.182 & 2.4535 $\pm$ 0.188 \\
& \Block[fill=gray!15]{1-}{}\textbf{ProtScape (DDPM)} & \textbf{0.1980 $\pm$ 0.003} & \textbf{0.1298 $\pm$ 0.004} & \textbf{0.1494 $\pm$ 0.003} & \textbf{0.4991 $\pm$ 0.067} & \textbf{0.5173 $\pm$ 0.064} \\
& \Block[fill=gray!15]{1-}{}\textbf{ProtScape (Flow Matching)} & 0.2775 $\pm$ 0.017 & 0.1979 $\pm$ 0.026 & 0.1934 $\pm$ 0.007 & 1.1373 $\pm$ 0.264 & 1.1389 $\pm$ 0.263 \\
\midrule
\multirow{4}{*}{\makecell{7LP1\\(Chain A)}}
& AlphaFlow & 1.2303 $\pm$ 0.017 & 0.7533 $\pm$ 0.013 & 0.9727 $\pm$ 0.013 & 3.8058 $\pm$ 0.153 & 3.7298 $\pm$ 0.140 \\
& MDGen & 0.4694 $\pm$ 0.014 & 0.3623 $\pm$ 0.014 & 0.2983 $\pm$ 0.009 & 0.8995 $\pm$ 0.180 & 1.1783 $\pm$ 0.111 \\
& \Block[fill=gray!15]{1-}{}\textbf{ProtScape (DDPM)} & \textbf{0.1916 $\pm$ 0.010} & \textbf{0.0847 $\pm$ 0.009} & \textbf{0.1718 $\pm$ 0.007} & \textbf{0.3511 $\pm$ 0.091} & \textbf{0.3594 $\pm$ 0.091} \\
& \Block[fill=gray!15]{1-}{}\textbf{ProtScape (Flow Matching)} & 0.2545 $\pm$ 0.022 & 0.1147 $\pm$ 0.011 & 0.2269 $\pm$ 0.022 & 0.4736 $\pm$ 0.177 & 0.4779 $\pm$ 0.177 \\
\midrule
\multirow{4}{*}{\makecell{7JFL\\(Chain C)}}
& AlphaFlow & 1.9544 $\pm$ 0.050 & 1.5941 $\pm$ 0.028 & 1.1305 $\pm$ 0.048 & 13.0873 $\pm$ 0.352 & 13.0587 $\pm$ 0.351 \\
& MDGen & 0.3431 $\pm$ 0.003 & 0.2088 $\pm$ 0.007 & 0.2721 $\pm$ 0.003 & \textbf{0.5640 $\pm$ 0.043} & \textbf{0.5820 $\pm$ 0.041} \\
& \Block[fill=gray!15]{1-}{}\textbf{ProtScape (DDPM)} & \textbf{0.2117 $\pm$ 0.006} & \textbf{0.1691 $\pm$ 0.006} & \textbf{0.1273 $\pm$ 0.003} & 0.7591 $\pm$ 0.070 & 0.7569 $\pm$ 0.071 \\
& \Block[fill=gray!15]{1-}{}\textbf{ProtScape (Flow Matching)} & 0.3030 $\pm$ 0.028 & 0.2326 $\pm$ 0.029 & 0.1936 $\pm$ 0.013 & 1.2047 $\pm$ 0.298 & 1.2046 $\pm$ 0.298 \\
\bottomrule
\end{NiceTabular}
}
\endgroup
\end{small}
\end{table*}

\noindent\textbf{Minimum Energy Path (NEB): } Finally, we evaluate minimum energy path interpolation using the proposed latent-space NEB formulation. Given two protein conformations, we optimize a sequence of representation-space points to approximate smooth low-energy transition pathways between endpoints. Table~\ref{tab:neb_results} reports structural validity and transition smoothness metrics for generated interpolation paths. Across all proteins, \modelname{} produces substantially lower Ramachandran outlier percentages, clash scores, MolProbity scores \cite{davis_molprobity_2007}, and consecutive-frame RMSD values than MDGen, indicating smoother and more physically realistic conformational transitions. These results demonstrate that organizing the learned conformational manifold according to both structure and potential energy enables physically meaningful interpolation through sparsely sampled regions of conformational space. Ablation studies of the minimum energy path generation are provided in Appendix~\ref{app:abl}.

\begin{table*}[!htbp]
\caption{\textbf{Minimum Energy Path using NEB.} Structural validity is measured via Ramachandran outliers, steric clash score, rotamer outliers, and MolProbity score. We also report $\Delta$RMSD (in \AA; lower is better). Values are reported as mean $\pm$ std.}
\label{tab:neb_results}
\centering
\setlength{\tabcolsep}{6pt}
\renewcommand{\arraystretch}{1.2}
\resizebox{0.90\textwidth}{!}{
\begin{NiceTabular}{llccccc}
\toprule
\textbf{Protein} & \textbf{Method}
& Rama. Outliers (\%)$\downarrow$
& Clash Score$\downarrow$
& Rotamer Outliers$\downarrow$
& MolProbity Score$\downarrow$
& $\Delta$RMSD (\AA)$\downarrow$ \\
\midrule

\multirow{2}{*}{\makecell{6H86\\(Chain A)}}
& MDGen
& 53.56 $\pm$ 11.87
& 397.89 $\pm$ 46.74
& 5.03 $\pm$ 0.11
& 16.30 $\pm$ 4.60
& 2.10 $\pm$ 0.56 \\

& \Block[fill=gray!15]{1-}{}\textbf{ProtScape}
& \textbf{1.70 $\pm$ 1.65}
& \textbf{99.27 $\pm$ 99.80}
& \textbf{0.00 $\pm$ 0.00}
& \textbf{2.44 $\pm$ 0.82}
& \textbf{0.20 $\pm$ 0.05} \\
\midrule

\multirow{2}{*}{\makecell{6P5H\\(Chain A)}}
& MDGen
& 43.26 $\pm$ 7.17
& 371.61 $\pm$ 38.88
& 4.85 $\pm$ 0.11
& 11.54 $\pm$ 3.16
& 1.70 $\pm$ 0.29 \\

& \Block[fill=gray!15]{1-}{}\textbf{ProtScape}
& \textbf{3.38 $\pm$ 2.16}
& \textbf{124.37 $\pm$ 108.04}
& \textbf{0.00 $\pm$ 0.00}
& \textbf{2.73 $\pm$ 0.67}
& \textbf{0.113 $\pm$ 0.028} \\
\midrule

\multirow{2}{*}{\makecell{6JV8\\(Chain A)}}
& MDGen
& 54.95 $\pm$ 8.15
& 371.04 $\pm$ 53.68
& 4.90 $\pm$ 0.14
& 12.10 $\pm$ 4.01
& 2.00 $\pm$ 0.38 \\

& \Block[fill=gray!15]{1-}{}\textbf{ProtScape}
& \textbf{5.57 $\pm$ 2.85}
& \textbf{154.15 $\pm$ 133.61}
& \textbf{0.00 $\pm$ 0.00}
& \textbf{2.81 $\pm$ 0.97}
& \textbf{0.114 $\pm$ 0.037} \\
\midrule

\multirow{2}{*}{\makecell{7LP1\\(Chain A)}}
& MDGen
& 42.65 $\pm$ 10.59
& 392.16 $\pm$ 63.69
& 4.90 $\pm$ 0.18
& 12.47 $\pm$ 4.99
& 1.50 $\pm$ 0.29 \\

& \Block[fill=gray!15]{1-}{}\textbf{ProtScape}
& \textbf{2.22 $\pm$ 1.64}
& \textbf{213.27 $\pm$ 126.70}
& \textbf{0.00 $\pm$ 0.00}
& \textbf{3.24 $\pm$ 0.36}
& \textbf{0.086 $\pm$ 0.050} \\
\midrule

\multirow{2}{*}{\makecell{7JFL\\(Chain C)}}
& MDGen
& 52.22 $\pm$ 13.20
& 326.45 $\pm$ 58.39
& 4.72 $\pm$ 0.20
& 10.34 $\pm$ 4.62
& 1.90 $\pm$ 0.40 \\

& \Block[fill=gray!15]{1-}{}\textbf{ProtScape}
& \textbf{2.51 $\pm$ 3.20}
& \textbf{183.63 $\pm$ 149.23}
& \textbf{0.00 $\pm$ 0.00}
& \textbf{2.92 $\pm$ 0.76}
& \textbf{0.109 $\pm$ 0.064} \\
\bottomrule

\end{NiceTabular}
}
\end{table*}

\noindent\textbf{Energy descent.} Unlike the latent NEB setting, which seeks a minimum-energy transition path between two specified endpoint conformations, energy descent instead begins from a single high-energy conformation and generates a directed trajectory toward a lower-energy basin on the conformational landscape. Such trajectories are useful for studying relaxation dynamics and aspects of protein folding, where proteins evolve toward energetically favorable and potentially more stable conformations. Importantly, the generated paths provide information about the sequence of conformational changes occurring during this descent process, rather than only recovering a final low-energy structure or interpolating between prescribed endpoints.

We evaluate the latent energy-descent generation mode by generating 10 energy-descent paths per method on each of five ATLAS proteins, and report three path-level metrics in Appendix~\ref{app:energy_descent} Table~\ref{tab:noiseless-rmsd-energy-success}: energy decrease success (fraction of paths with $E_{\text{end}} < E_{\text{start}}$), consistent descent (fraction of paths whose net descent exceeds half of their total energy variation), and pairwise RMSD difference between the generated path and the reference ensemble (lower is better).

ProtScape produces paths that decrease energy most reliably and most consistently. Aggregated across all 50 paths per method, 94.0\% of ProtScape paths reduce energy end-to-end and 90.0\% descend consistently, well above the next-best alternatives at 80.0\% and 40.0\% respectively. On pairwise RMSD difference, ProtScape is essentially tied with the best alternative (0.6705 vs.\ 0.6526) and well below the remaining methods, indicating that the structural spread of its paths matches the reference ensemble as closely as any method tested.

\section{Conclusion}

In this work, we introduced ProtScape, a generative framework for learning structured and energy-aware representations of protein conformational landscapes from molecular dynamics trajectories. Rather than operating directly in ambient coordinate space, ProtScape learns a multiscale latent manifold that organizes protein conformations according to both geometric structure and physical energy. This representation enables multiple complementary modes of conformational generation, including ensemble upsampling, minimum-energy path interpolation, and energy-guided descent toward lower-energy states. Across proteins from the ATLAS dataset, ProtScape generated conformational ensembles and transition pathways with favorable structural validity, smoother energy organization, and improved distributional agreement relative to competing approaches. Together, these results suggest that explicitly learning structured latent conformational landscapes provides a promising direction for scalable and physically meaningful modeling of protein dynamics.

\paragraph{Limitations and Future Work.} While ProtScape demonstrates strong performance across multiple generative tasks, direct comparisons for both energy-guided descent generation and minimum-energy path generation remain limited because existing generative approaches typically do not explicitly address these problems. In addition, ProtScape currently learns conformational representations for a single protein system at a time, whereas several recent generative approaches are trained across large collections of protein trajectories. However, this protein-specific representation learning enables explicit organization of the latent conformational landscape by both structure and energy, supporting tasks such as energy-guided descent and minimum-energy path generation that are difficult to perform in more generalized generative settings. Furthermore, the energy signal used for regularization and guidance may be unreliable. Computing potential energy with OpenMM and AMBER on single decoded frames introduces several concerns: the energy is typically evaluated without explicit solvent despite ATLAS trajectories being run in explicit solvent. Future work should explore more robust energy estimation, such as using solvent-aware potentials, energy minimization prior to scoring, or learned surrogate energy models. Finally, our experiments focus primarily on proteins from the ATLAS dataset, and future work should evaluate scalability to larger molecular systems, multi-chain complexes, intrinsically disordered proteins, as well as protein folding trajectories and long-timescale conformational processes.

\section*{Acknowledgments}

D.B. received funding from the Yale - Boehringer Ingelheim Biomedical Data Science Fellowship. M.P. acknowledges funding from The National Science Foundation under grant number OIA-2242769. S.K. and M.P. also acknowledge funding from NSF-DMS Grant No. 2327211.



\newpage

\bibliographystyle{unsrt}  
\bibliography{references}
\newpage

\appendix
\input{appendix}

\end{document}

%% file: appendix.tex
\section{Related Work}\label{app: related}

The problem of characterizing protein dynamics can be broadly categorized into two computational challenges: generating equilibrium conformational ensembles and identifying minimum energy pathways (MEPs) between metastable states. While traditional molecular dynamics (MD) simulations offer a principled solution to both, their computational cost has motivated a surge in data-driven and physics-informed learning approaches.

\textbf{Transition Path Sampling and Interpolation.} Classical approaches to path sampling often rely on geometric heuristics to bridge the gap between start and end states. For instance, \cite{nguyen2018generating} proposed a multi-stage pipeline combining As-Rigid-As-Possible (ARAP) interpolation with collision detection and Nudged Elastic Band (NEB) optimization. Similarly, the ``cold-inbetweening'' algorithm \cite{yorke2025calculation} generates trajectories for transport proteins by interpolating between experimentally determined end-states to hypothesize functional mechanisms. While computationally inexpensive, these geometry-centric methods often lack explicit thermodynamic grounding and treat path generation as a discrete interpolation problem rather than a search on a learned manifold.

More recent works have integrated deep learning with physical constraints to discover transition paths. \cite{pang2025deeppath} introduced DeepPath, an active learning framework that uses a force field oracle to iteratively refine generated pathways, demonstrating success on complex systems like the BAM complex. Alternatively, \cite{petersen2025pinn} reformulated path generation as a continuous optimization problem using Physics-Informed Neural Networks (PINNs). By representing paths as implicit neural functions optimized against differentiable force fields, they achieved significant speedups over explicit solvent MD. However, both approaches typically solve for a specific transition path \textit{ab initio} or require costly active learning loops during inference. In contrast, \modelname{} leverages a pre-trained encoder–decoder model trained with an energy-based regularizer, yielding a latent space that faithfully represents the underlying data manifold. This allows us to perform NEB optimization in a low-dimensional space where topological constraints are implicitly handled by the decoder, rather than optimizing in the high-dimensional atomic space.

\textbf{Generative Modeling of Ensembles and Dynamics.} Generative models have recently shown great promise in sampling equilibrium distributions. AlphaFlow \cite{jing2024alphafold} adapts high-accuracy structure predictors like AlphaFold via flow matching to sample from MD-derived ensembles, effectively trading static precision for conformational diversity. In the domain of trajectory modeling, MDGen \cite{jing2024generative} treats molecular dynamics as a high-dimensional tensor generation problem using stochastic interpolants, enabling tasks such as trajectory upsampling and inpainting.

While powerful, these methods largely focus on sampling distributions rather than characterizing the connectivity of the energy landscape. An exception is JAMUN \cite{daigavane2025jamun}, which employs Walk-Jump sampling to perform Langevin dynamics in a smoothed, noised conformational space. However, JAMUN operates by perturbing the high-dimensional atomistic space. \textbf{\modelname{}} distinguishes itself by learning a compact, geometrically invariant latent manifold regularized by potential energy. This unified representation enables us to support both equilibrium sampling (via latent diffusion) and energetic path finding (via latent NEB) within a single framework, bridging the gap between ensemble generation and transition path analysis.

\section{Generation}
\label{app:generation}
\subsection{Latent Generative Dynamics}

The latent representation described in Section \ref{sec:train}  extracts
invariant geometric features from the data and organizes them into a structured manifold. In this section, we provide further details on how we leverage this representation to model protein dynamics.
We perform generative modeling directly in the low-dimensional latent space $\mathcal{Z}$, allowing us to efficiently traverse the landscape and sample physically valid structures without the computational cost of atomistic simulations~\cite{sun2024geometry,liao2025rnagenscape}.

\subsubsection{Ensemble Generation via Latent Diffusion}
\label{sec:latent_ddpm}

To capture the diverse conformational ensemble of the protein system, we employ a Denoising Diffusion Probabilistic Model (DDPM)~\cite{ho2020denoising}.
Building on recent success in modeling biological sequences via latent Langevin dynamics~\cite{sun2024geometry,liao2025rnagenscape}, we adapt this manifold-centric approach to capture the continuous conformational changes of proteins.
By training the diffusion model on the \modelname{} latent representations, we learn the probability density supported on the conformational manifold, allowing us to sample thermodynamically valid structures that follow the Boltzmann distribution.

We define a forward diffusion process that gradually adds Gaussian noise to the latent code $\mathbf{z}_0$:
\begin{equation*}
\smash{q(\mathbf{z}_t | \mathbf{z}_{t-1}) = \mathcal{N}(
\sqrt{1-\beta_t}\mathbf{z}_{t-1}, \beta_t I).}
\end{equation*}
The reverse process $p_\theta(\mathbf{z}_{t-1} | \mathbf{z}_t)$ is parameterized by a neural network $\boldsymbol{\epsilon}_\theta(\mathbf{z}_t, t)$ trained to predict the noise component as described in Algorithm \ref{alg:ddpm}.
In practice, to facilitate stable density estimation, we train the diffusion model on a compressed representation obtained via a lightweight autoencoder (see Appendix \ref{app:implementation} for architecture details).
This projection focuses the generative modeling on the principal manifold of the system, capturing the dominant modes of conformational variation.

\begin{algorithm}[tb]
   \caption{Latent DDPM Sampling}
   \label{alg:ddpm}
\begin{algorithmic}
   \STATE {\bfseries Input:} Denoiser $\boldsymbol{\epsilon}_\theta$, Decoder $g_\text{dec}$, Steps $T$.
   \STATE {\bfseries Sample:} $\mathbf{z}_T \sim \mathcal{N}(\mathbf{0}, I)$
   \FOR{$t=T$ {\bfseries down to} $1$}
   \STATE $\mathbf{z} \sim \mathcal{N}(\mathbf{0}, I)$ if $t > 1$ else $\mathbf{0}$
   \STATE $\mathbf{z}_{t-1} \leftarrow \frac{1}{\sqrt{\alpha_t}} \left( \mathbf{z}_t - \frac{1-\alpha_t}{\sqrt{1-\bar{\alpha}_t}} \boldsymbol{\epsilon}_\theta(\mathbf{z}_t, t) \right) + \sigma_t \mathbf{z}$
   \ENDFOR
   \STATE \textbf{Output:} $\mathbf{x}_{\text{gen}} \leftarrow g_\text{dec}(\mathbf{z}_0)$
\end{algorithmic}
\end{algorithm}

\subsubsection{Energy-Guided Generative Dynamics:}\label{sec:energyguidedgen}

The latent diffusion model in Section~\ref{sec:latent_ddpm} samples the ensemble at the temperature of the MD simulation data. However, many downstream applications require informed access to specific regions and routes on the energy landscape rather than unconstrained equilibrium sampling. Tasks like docking and comparison to crystal structures need the dominant low-energy conformation of a basin, which corresponds to a low-temperature query. Applications like drug discovery often require rare conformations such as cryptic binding pockets, which are suppressed at physiological temperature and undersampled by MD. Beyond sampling individual states, many problems require understanding how conformations transition between metastable states. These transitions are governed by the structure of the energy landscape between endpoints, and in particular by minimum energy paths (MEPs), which characterize transition mechanisms and energy barriers. Such pathways are inaccessible from equilibrium sampling alone, as they lie in low-probability regions of the distribution. As a result, a single equilibrium trajectory is not sufficient for both targeted low-energy sampling and transition mechanism discovery.

To address this, we introduce energy-guided generative dynamics in latent space, enabling physical energy gradients to guide both low-energy conformational sampling and minimum energy path estimation on the learned manifold.

\noindent\textbf{Mode 1: Thermodynamic sampling via Latent Langevin Dynamics}
Here, we describe a sampling procedure on the latent manifold based on Langevin dynamics, designed to recover low-energy conformations while enabling informed exploration of the low-energy landscape. This is made practical by the structure of our latent manifold: the energy-based smoothness regularization in Section~\ref{sec:train} trains the encoder so that nearby latent points have similar energies, which means energy varies smoothly along the manifold rather than jumping erratically with small latent steps. Energy gradients in latent space are therefore well-behaved, and short Langevin steps move the conformation in directions that are both physically meaningful and consistent with the learned distribution of valid structures. Sampling at temperatures other than the training temperature then requires only the learned manifold and a differentiable energy function.

\noindent\textbf{Latent Langevin SDE.}
We let $g_\text{enc}$
and $g_\text{dec}$ 
be the trained encoder and decoder, 
 and assume access to a differentiable potential $U(\mathbf{x})$, taken in our experiments as the AMBER14 force field~\cite{maier2015ff14sb} evaluated on the decoded coordinates. We evolve $\mathbf{z}_t$ under overdamped Langevin dynamics with the temperature-scaled potential
\begin{equation}
    V_T(\mathbf{z}) \;=\; \frac{1}{k_B T}\, U\bigl(g_\text{dec}(\mathbf{z})\bigr),
\end{equation}
where $T$ is the simulation temperature and $k_B$ is the Boltzmann constant. A projected Euler-Maruyama discretization gives
\begin{align}
    \tilde{\mathbf{z}}_{t+1} &= \mathbf{z}_t - \tfrac{\eta}{2}\, \nabla_{\mathbf{z}} V_T(\mathbf{z}_t) + \sqrt{\eta}\, \boldsymbol{\epsilon}_t, \quad \boldsymbol{\epsilon}_t \sim \mathcal{N}(\mathbf{0}, I), \label{eqn:langevin_update} \\
    \mathbf{z}_{t+1} &= \mathcal{P}(\tilde{\mathbf{z}}_{t+1}), \label{eqn:manifold_projection}
\end{align}
where $\eta$ is the step size and $\mathcal{P}$ is the manifold projector described below.

\noindent\textbf{Drift via decoder pullback.}
We compute the drift by pulling the ambient force back through the decoder:
\begin{equation}
    \label{eqn:jacobian_pullback}
    \nabla_{\mathbf{z}} V_T(\mathbf{z}) \;=\; \frac{1}{k_B T}\, J_{g_\text{dec}}^\top(\mathbf{z}) \cdot \nabla_{\mathbf{x}} U(\mathbf{x}) \big|_{\mathbf{x} = g_\text{dec}(\mathbf{z})},
\end{equation}
where $J_{g_\text{dec}}(\mathbf{z}) \in \mathbb{R}^{D \times d}$ is the decoder Jacobian. The gradient is computed entirely from the decoder and the energy function, so the dynamics can start from any latent initialization without needing samples near the target.

\noindent\textbf{Learned manifold projector.}
At high temperature, the noise term is large, and the energy gradient can push the trajectory off the manifold of valid protein conformations. To prevent this, we use a learned projector $\mathcal{P}: \mathbb{R}^d \to \mathcal{Z}$ implemented as a denoising autoencoder (DAE)~\cite{vincent2008extracting}. We train the DAE by adding Gaussian noise to the latent codes of training conformations and learning to recover the originals. This objective makes $\mathcal{P}$ approximate the gradient of the log-density of the training distribution~\cite{alain2014regularized}, pulling perturbed points back toward high-density regions. Applying $\mathcal{P}$ after each Langevin step keeps trajectories on the learned manifold while still allowing energy-driven motion within it.

\noindent\textbf{Temperature as a control parameter.}
Temperature $T$ enters Eq.~\ref{eqn:jacobian_pullback} only as an inverse weight on the gradient, but it controls two regimes. At low $T$, the drift dominates the noise and the dynamics behave as energy minimization. Trajectories settle into the local energy basin and recover its dominant low-energy structure, the kind used in docking and crystal-structure comparison. At high $T$, the noise dominates the drift, and trajectories cross local barriers to visit conformations with low Boltzmann weight at physiological temperature, including rare states such as cryptic pockets that MD undersamples. The same components, namely the decoder, energy function, and projector, support both regimes. We report sampling at multiple temperatures in Section~\ref{sec: results}.

\noindent\textbf{Mode 2: Energy-guided Conformational Interpolation via Latent NEB}
 Next, we describe a path optimization mechanism based on the Nudged Elastic Band (NEB) method to operate on the latent manifold to recover transition pathways between metastable states. To identify the Minimum Energy Path (MEP) connecting a start and target conformation, we employ the Nudged Elastic Band (NEB) method~\cite{draxler2018essentially} adapted for the latent manifold~(Algorithm \ref{alg:neb}).

We define the potential energy $U(\mathbf{z}) \equiv U_{\text{phys}}(g_\text{dec}(\mathbf{z}))$ as the physical energy of the decoded structure, computed using the OpenMM toolkit~\cite{eastman2023openmm} with the AMBER14 force field~\cite{maier2015ff14sb}.
To compute the gradient on the latent manifold $\nabla_{\mathbf{z}} U$, we apply the chain rule via vector-Jacobian products:
\begin{equation*}
    \nabla_{\mathbf{z}} U = \left( \nabla_{\mathbf{x}} U_{\text{phys}} \right)^\top \cdot J_{g_\text{dec}}(\mathbf{z}).
\end{equation*}
We represent the transition path as a discretized chain of $K$ latent states $\{\mathbf{z}^{(0)}, \dots, \mathbf{z}^{(K)}\}$.
Unlike standard NEB, which relies on sensitive spring constants to maintain path integrity, we utilize an iterative pivot redistribution scheme.
At each optimization step, we first re-parameterize the path to ensure equal arc-length spacing between pivots.
We then update the pivots using the projected gradient component perpendicular to the path tangent $\hat{\boldsymbol{\tau}}_i$:
\begin{equation*}
    F^\perp_{i} = - \left( \nabla_{\mathbf{z}} U(\mathbf{z}_i) - \left( \nabla_{\mathbf{z}} U(\mathbf{z}_i) \cdot \hat{\boldsymbol{\tau}}_i \right) \hat{\boldsymbol{\tau}}_i \right).
\end{equation*}
This decoupling ensures the path converges to the ``valley floor'' of the energy landscape while the redistribution step prevents pivots from collapsing into local minima.

\begin{algorithm}[tb]
   \caption{Latent NEB with Arc-Length Redistribution}
   \label{alg:neb}
\begin{algorithmic}
   \STATE {\bfseries Input:} Decoder $g_\text{dec}$, Energy $U(\mathbf{x})$, Start $\mathbf{z}_0$, End $\mathbf{z}_K$, Pivots $K$, Steps $N$, Step Size $\eta$.
   \STATE {\bfseries Initialize:} Linear interpolation path $\{\mathbf{z}^{(i)}\}_{i=0}^K$
   \FOR{$n=1$ {\bfseries to} $N$}
       \STATE $\{\mathbf{z}^{(i)}\} \leftarrow \text{RedistributePivots}(\{\mathbf{z}^{(i)}\})$
       
       \FOR{$i=1$ {\bfseries to} $K-1$}
           \STATE $\boldsymbol{\tau}_i \leftarrow \mathbf{z}^{(i+1)} - \mathbf{z}^{(i-1)}$
           \STATE $\hat{\boldsymbol{\tau}}_i \leftarrow \boldsymbol{\tau}_i / \| \boldsymbol{\tau}_i \|$
           
           \STATE $\mathbf{x}_i \leftarrow g_\text{dec}(\mathbf{z}^{(i)})$
           \STATE $\mathbf{g}_i \leftarrow (\nabla_{\mathbf{x}} U(\mathbf{x}_i))^\top J_{g}(\mathbf{z}^{(i)})$
           \STATE $F^\perp_{i} \leftarrow - (\mathbf{g}_i - (\mathbf{g}_i \cdot \hat{\boldsymbol{\tau}}_i) \hat{\boldsymbol{\tau}}_i)$
           
           \STATE $\mathbf{z}^{(i)} \leftarrow \mathbf{z}^{(i)} + \eta \frac{F^\perp_{i}}{\|F^\perp_{i}\| + \epsilon}$
       \ENDFOR
   \ENDFOR
   \STATE \textbf{Output:} Minimum Energy Path $\{\mathbf{z}^{(i)}\}_{i=0}^K$
\end{algorithmic}
\end{algorithm}
\vspace{-10pt}

\section{Implementation Details}
\label{app:implementation}

We provide the architectural specifications and hyperparameters for the generative framework. All models were implemented in PyTorch and trained on a single NVIDIA A100 GPU.

\subsection{Data splits and baseline training}
\label{app:data}

\modelname{} learns a representation per protein, so each model is trained on the molecular dynamics trajectory of a single protein from the ATLAS dataset~\cite{10.1093/nar/gkad1084}, and each stage of training splits the frames of that trajectory into a training and a test set. The autoencoder of Section~\ref{sec:train} holds out 200 frames for testing and trains on the remainder. The latent generative models of Section~\ref{sec:latent_dynamics} use a random 80/20 train/test split of the latent codes. We do not use k-fold or cross-validation, since each model is fit to a single protein's trajectory.

In their original form, AlphaFlow and MDGen are trained across collections of protein trajectories. We instead retrain each of them separately on each target protein, so that every method receives the same single-protein training data as \modelname{}.

\subsection{ProtSCAPE}
\paragraph{Architecture.} The encoder consists of stacked EGNN layers with a hidden representation of $128$ dimensions and 3 layers of deep wavelet transform using $J = 4$ diffusion scales. The output of the EGNN and deep wavelet transform is passed into a bottleneck module implemented as a three-layer MLP with GeLU activations, producing a latent representation of $128$ dimensions. The decoder is a three-layer MLP with GeLU activations that decodes the latent representations directly to three-dimensional Cartesian coordinates for each residue.

\paragraph{Alignment Loss.} Reconstruction of Cartesian coordinates using the Kabsch-aligned mean-squared error loss. Since proteins are defined up to global rigid motions, we compute the mean-squared error loss after optimal rigid alignment of $\hat{\mathbf{c}}^{(t)}$ to $\mathbf{c}^{(t)}$. We first define a node validity mask $m_i \in {\{0,1}\}$ where the value is 1 if node $i$ corresponds to a valid atom. Then the masked centroids of the predicted and ground truth coordinates are given by:
\begin{equation*}
    \bar{\hat{\mathbf{c}}} = \frac{1}{\sum_{i=1}^{N} m_i}\sum_{i=1}^{N} m_i \hat{\mathbf{c}}_i, \qquad \bar{\mathbf{c}} =
\frac{1}{\sum_{i=1}^{N} m_i}
\sum_{i=1}^{N} m_i \mathbf{c}_i.
\end{equation*}

We next center the coordinates and apply the mask and define $\mathbf{p}_i = m_i \left(\hat{\mathbf{c}}_i - \bar{\hat{\mathbf{c}}}\right)$ and $\mathbf{q}_i = m_i \left(\mathbf{c}_i - \bar{\mathbf{c}}\right)$. The masked covariance matrix $H\in\mathbf{R}^{3\times3}$ is defined as: ${H = \sum_{i=1}^{N} \mathbf{p}_i^\top \mathbf{q}_i.}$.

Next, we compute the Singular Value Decomposition of $H$ as $H = U {\Sigma} V^\top$, which gives us the optimal rotation matrix $R^\star = V U^\top$. The Kabsch aligned mean-squared error coordinate loss is now defined as
\begin{equation*}
    \mathcal{L}_c = \frac{1}{\sum_{i=1}^{N} m_i}\sum_{i=1}^{N} m_i\left\|\mathrm{Align}(\hat{\mathbf{c}}, \mathbf{c})_i - \mathbf{c}_i
\right\|_2^2.
\end{equation*}

where $\mathrm{Align}(\hat{\mathbf{c}}, \mathbf{c}) = (\hat{\mathbf{c}} - \bar{\hat{\mathbf{c}}}) R^\star + \bar{\mathbf{c}}$. This loss is invariant to global translations and rotations and matches the exact alignment procedure used in training.

\paragraph{Training.} Models are trained for 800 epochs with batch size 64 using AdamW optimizer with a learning rate of $10^{-3}$. Training minimizes a weighted combination of loss terms. Loss weights are set to $70.0$ for the Kabsch-aligned coordinate reconstruction term $\lambda_c$ and $20.0$ for the energy loss $\lambda_e$.

\subsection{Latent Manifold Compression}
To facilitate robust density estimation, we project the primary geometric embeddings ($\mathbb{R}^{128}$) onto a principal manifold ($\mathbb{R}^{16}$) using a lightweight compression network.
\paragraph{Architecture.} The encoder is a symmetric MLP with hidden dimensions $[256, 128]$ mapping to a $16$-dimensional bottleneck. The decoder mirrors this structure to reconstruct the original embedding. Each linear block includes Layer Normalization, ReLU activation, and dropout ($p=0.1$).
\paragraph{Training.} We optimize the network for 200 epochs using AdamW (batch size 256, weight decay $10^{-5}$). The learning rate follows a cosine decay schedule from $10^{-3}$ to $10^{-5}$. The objective is the Mean Squared Error (MSE) of the reconstruction.

\subsubsection{Latent Diffusion Model}
The generative model is a Denoising Diffusion Probabilistic Model (DDPM) trained on the compressed latent subspace.
\paragraph{Denoising Network.} The score network $\boldsymbol{\epsilon}_\theta(\mathbf{z}_t, t)$ is a Residual MLP. It processes the noisy latent state $\mathbf{z}_t$ and a sinusoidal timestep embedding (projected to 64 dimensions). The backbone comprises four residual blocks with hidden dimensions $[512, 1024, 1024, 512]$, utilizing Layer Normalization and GELU activations.
\paragraph{Diffusion Process.} We employ a discrete-time diffusion process with $T=1000$ steps. To improve sampling quality, we utilize the cosine noise schedule~\citep{nichol2021improved} for $\beta_t$, clamped to $[10^{-4}, 0.9999]$. Inputs are standardized to $\mathcal{N}(0, 1)$ prior to training.
\paragraph{Optimization.} The model is trained for 1000 epochs (batch size 256) using AdamW (lr $10^{-4}$, weight decay $10^{-4}$). We employ a cosine learning rate schedule with a minimum learning rate of $10^{-6}$. We apply Exponential Moving Average (EMA) with a decay of $0.9999$ to the model weights during inference to stabilize the generated trajectories.

\subsubsection{Latent Flow Matching Model}
The generative model is a rectified flow matching model trained on the compressed latent subspace. As in the latent diffusion baseline, protein representations are first compressed from 128 dimensions to a 16-dimensional latent space using the autoencoder, and the flow model is trained on standardized 16-dimensional latents.

\paragraph{Velocity Network.}
The flow model parameterizes a time-dependent velocity field $\mathbf{v}_\theta(\mathbf{z}_t, t)$ using a Residual MLP. The input is the interpolated latent state $\mathbf{z}_t$ concatenated with a sinusoidal time embedding projected to 64 dimensions. The MLP backbone uses hidden dimensions $[512, 1024, 1024, 512]$ with Layer Normalization, GELU activations, and dropout rate $0.1$. Residual blocks are used with matching hidden dimensions.

\paragraph{Flow Matching Objective.}
We train the model with the rectified flow objective. For each data latent $\mathbf{z}_1$ and Gaussian noise sample $\mathbf{z}_0 \sim \mathcal{N}(0, I)$, we sample $t \sim \mathcal{U}(0,1)$ and construct the linear interpolation
\[
\mathbf{z}_t = (1-t)\mathbf{z}_0 + t\mathbf{z}_1.
\]
The target velocity is
\[
\mathbf{u}_t = \mathbf{z}_1 - \mathbf{z}_0,
\]
and the model is trained by minimizing
\[
\mathcal{L}_{\mathrm{FM}} =
\mathbb{E}_{\mathbf{z}_1, \mathbf{z}_0, t}
\left[
\left\|
\mathbf{v}_\theta(\mathbf{z}_t,t) - \mathbf{u}_t
\right\|_2^2
\right].
\]
Inputs are standardized to zero mean and unit variance prior to training.

\paragraph{Sampling.}
Samples are generated by integrating the learned ODE
\[
\frac{d\mathbf{z}_t}{dt} = \mathbf{v}_\theta(\mathbf{z}_t,t)
\]
from $\mathbf{z}_0 \sim \mathcal{N}(0,I)$ to $t=1$. We use 1000 integration steps with Heun's method, and clamp latent values to $[-3.5, 3.5]$ during sampling for stability. Generated 16-dimensional latents are denormalized and decoded through the autoencoder back to the original 128-dimensional representation.

\paragraph{Optimization.}
The model is trained for 1000 epochs with batch size 256 using AdamW with learning rate $10^{-4}$ and weight decay $10^{-4}$. Gradients are clipped to norm 1.0. We use a cosine learning rate schedule with minimum learning rate $10^{-6}$. Exponential Moving Average (EMA) is applied to model weights with decay $0.999$, and EMA weights are used when saving checkpoints for inference.

\section{Ensemble Generation}
\label{app:figs}

Here, we evaluate both dynamical agreement and structural validity across multiple protein datasets for ensemble generation. As shown in Table~\ref{tab:dynamics-quality-fm}, ProtSCAPE achieves consistently strong agreement with reference MD dynamics, obtaining near-perfect RMSF correlations while maintaining low pairwise RMSD discrepancies across proteins. Table~\ref{tab:structural-quality-fm} further demonstrates that generated conformations remain physically realistic, with low Ramachandran outlier rates, clash scores, rotamer outliers, and MolProbity scores relative to competing generative baselines.

\begin{table*}[!htbp]
\caption{\textbf{Dynamic agreement metrics across protein datasets.}
Higher RMSF correlation is better. Lower pairwise RMSD percent difference is better. Best result in each category is bolded.}
\label{tab:dynamics-quality-fm}
\centering
\setlength{\tabcolsep}{6pt}
\begin{small}
\begingroup
\renewcommand{\arraystretch}{1.2}
\adjustbox{max width=\textwidth}{
\begin{NiceTabular}{llcc}
\toprule
\textbf{Protein} & \textbf{Method}
& RMSF Corr.$\uparrow$
& Pairwise RMSD$\downarrow$ \\
\midrule

\multirow{4}{*}{\makecell{6H86\\(Chain A)}}
& AlphaFlow & 0.8965 $\pm$ 0.013 & 0.7875 $\pm$ 0.004 \\
& MDGen & 0.9964 $\pm$ 0.000 & 0.1141 $\pm$ 0.018 \\
& \Block[fill=gray!15]{1-}{}\textbf{ProtSCAPE (DDPM)} & \textbf{0.9997 $\pm$ 0.000} & \textbf{0.0664 $\pm$ 0.010} \\
& \Block[fill=gray!15]{1-}{}\textbf{ProtSCAPE (Flow Matching)} & 0.9992 $\pm$ 0.001 & 0.0706 $\pm$ 0.022 \\
\midrule

\multirow{4}{*}{\makecell{6P5H\\(Chain A)}}
& AlphaFlow & 0.9980 $\pm$ 0.000 & 0.2234 $\pm$ 0.006 \\
& MDGen & 0.9979 $\pm$ 0.000 & 0.1082 $\pm$ 0.016 \\
& \Block[fill=gray!15]{1-}{}\textbf{ProtSCAPE (DDPM)} & \textbf{0.9998 $\pm$ 0.000} & \textbf{0.0403 $\pm$ 0.015} \\
& \Block[fill=gray!15]{1-}{}\textbf{ProtSCAPE (Flow Matching)} & 0.9993 $\pm$ 0.000 & 0.0545 $\pm$ 0.003 \\
\midrule

\multirow{4}{*}{\makecell{6JV8\\(Chain A)}}
& AlphaFlow & 0.9971 $\pm$ 0.000 & 0.3845 $\pm$ 0.007 \\
& MDGen & 0.9986 $\pm$ 0.000 & \textbf{0.0332 $\pm$ 0.011} \\
& \Block[fill=gray!15]{1-}{}\textbf{ProtSCAPE (DDPM)} & \textbf{0.9994 $\pm$ 0.000} & 0.0458 $\pm$ 0.012 \\
& \Block[fill=gray!15]{1-}{}\textbf{ProtSCAPE (Flow Matching)} & 0.9992 $\pm$ 0.000 & 0.0719 $\pm$ 0.006 \\
\midrule

\multirow{4}{*}{\makecell{7LP1\\(Chain A)}}
& AlphaFlow & 0.9548 $\pm$ 0.003 & 0.4331 $\pm$ 0.012 \\
& MDGen & 0.9932 $\pm$ 0.000 & \textbf{0.0293 $\pm$ 0.010} \\
& \Block[fill=gray!15]{1-}{}\textbf{ProtSCAPE (DDPM)} & \textbf{0.9988 $\pm$ 0.000} & 0.0701 $\pm$ 0.008 \\
& \Block[fill=gray!15]{1-}{}\textbf{ProtSCAPE (Flow Matching)} & 0.9970 $\pm$ 0.000 & 0.0880 $\pm$ 0.006 \\
\midrule

\multirow{4}{*}{\makecell{7JFL\\(Chain C)}}
& AlphaFlow & 0.9888 $\pm$ 0.002 & 0.4351 $\pm$ 0.013 \\
& MDGen & 0.9985 $\pm$ 0.000 & 0.0448 $\pm$ 0.008 \\
& \Block[fill=gray!15]{1-}{}\textbf{ProtSCAPE (DDPM)} & \textbf{0.9998 $\pm$ 0.000} & \textbf{0.0426 $\pm$ 0.009} \\
& \Block[fill=gray!15]{1-}{}\textbf{ProtSCAPE (Flow Matching)} & 0.9996 $\pm$ 0.000 & 0.0663 $\pm$ 0.006 \\
\bottomrule

\end{NiceTabular}
}
\endgroup
\end{small}
\end{table*}

\begin{table*}[!htbp]
\caption{\textbf{Structural quality metrics across protein datasets.}
Lower is better for all reported metrics. Best result in each category is bolded.}
\label{tab:structural-quality-fm}
\centering
\setlength{\tabcolsep}{5pt}
\begin{small}
\begingroup
\renewcommand{\arraystretch}{1.2}
\adjustbox{max width=\textwidth}{
\begin{NiceTabular}{llcccc}
\toprule
\textbf{Protein} & \textbf{Method}
& Rama. Outliers$\downarrow$
& Clash Score$\downarrow$
& Rotamer Outliers$\downarrow$
& MolProbity Score$\downarrow$ \\
\midrule

\multirow{4}{*}{\makecell{6H86\\(Chain A)}}
& AlphaFlow & 0.9864 $\pm$ 0.498 & \textbf{15.1059 $\pm$ 5.596} & 5.9902 $\pm$ 0.947 & 2.2979 $\pm$ 0.181 \\
& MDGen & 29.9647 $\pm$ 9.691 & 159.3615 $\pm$ 35.461 & 3.9302 $\pm$ 2.568 & 3.9921 $\pm$ 0.277 \\
& \Block[fill=gray!15]{1-}{}\textbf{ProtSCAPE (DDPM)} & \textbf{0.5639 $\pm$ 0.945} & 18.2541 $\pm$ 18.416 & \textbf{0.0000 $\pm$ 0.000} & \textbf{1.6399 $\pm$ 0.488} \\
& \Block[fill=gray!15]{1-}{}\textbf{ProtSCAPE (Flow Matching)} & 0.8560 $\pm$ 1.220 & 23.3332 $\pm$ 23.665 & \textbf{0.0000 $\pm$ 0.000} & 1.7706 $\pm$ 0.541 \\
\midrule

\multirow{4}{*}{\makecell{6P5H\\(Chain A)}}
& AlphaFlow & \textbf{0.4074 $\pm$ 0.666} & 16.7337 $\pm$ 5.121 & 1.4546 $\pm$ 0.594 & 1.9246 $\pm$ 0.227 \\
& MDGen & 10.3405 $\pm$ 3.443 & 164.8066 $\pm$ 24.569 & 4.8033 $\pm$ 2.032 & 3.9622 $\pm$ 0.192 \\
& \Block[fill=gray!15]{1-}{}\textbf{ProtSCAPE (DDPM)} & 1.2644 $\pm$ 1.122 & \textbf{16.2129 $\pm$ 12.608} & \textbf{0.0000 $\pm$ 0.000} & \textbf{1.8516 $\pm$ 0.467} \\
& \Block[fill=gray!15]{1-}{}\textbf{ProtSCAPE (Flow Matching)} & 1.6097 $\pm$ 1.208 & 20.7685 $\pm$ 14.488 & \textbf{0.0000 $\pm$ 0.000} & 2.0163 $\pm$ 0.439 \\
\midrule

\multirow{4}{*}{\makecell{6JV8\\(Chain A)}}
& AlphaFlow & \textbf{1.1344 $\pm$ 1.157} & 19.0241 $\pm$ 6.913 & 0.6411 $\pm$ 0.869 & 1.9835 $\pm$ 0.248 \\
& MDGen & 17.8680 $\pm$ 5.696 & 161.0425 $\pm$ 26.028 & 3.5839 $\pm$ 2.222 & 3.9382 $\pm$ 0.230 \\
& \Block[fill=gray!15]{1-}{}\textbf{ProtSCAPE (DDPM)} & 1.2751 $\pm$ 1.548 & \textbf{16.4361 $\pm$ 15.736} & \textbf{0.0000 $\pm$ 0.000} & \textbf{1.6672 $\pm$ 0.599} \\
& \Block[fill=gray!15]{1-}{}\textbf{ProtSCAPE (Flow Matching)} & 1.7862 $\pm$ 1.772 & 22.2932 $\pm$ 17.747 & \textbf{0.0000 $\pm$ 0.000} & 1.9012 $\pm$ 0.564 \\
\midrule

\multirow{4}{*}{\makecell{7LP1\\(Chain A)}}
& AlphaFlow & \textbf{0.9238 $\pm$ 1.432} & 30.2362 $\pm$ 11.954 & 6.8952 $\pm$ 3.474 & 2.7960 $\pm$ 0.317 \\
& MDGen & 17.7048 $\pm$ 8.142 & 214.6923 $\pm$ 43.647 & 8.9003 $\pm$ 4.138 & 4.3451 $\pm$ 0.236 \\
& \Block[fill=gray!15]{1-}{}\textbf{ProtSCAPE (DDPM)} & 1.3025 $\pm$ 2.148 & \textbf{23.9564 $\pm$ 23.546} & \textbf{0.0000 $\pm$ 0.000} & \textbf{1.9311 $\pm$ 0.691} \\
& \Block[fill=gray!15]{1-}{}\textbf{ProtSCAPE (Flow Matching)} & 2.0011 $\pm$ 2.738 & 30.1187 $\pm$ 27.226 & \textbf{0.0000 $\pm$ 0.000} & 2.0996 $\pm$ 0.693 \\
\midrule

\multirow{4}{*}{\makecell{7JFL\\(Chain C)}}
& AlphaFlow & \textbf{0.1878 $\pm$ 0.910} & 20.4902 $\pm$ 9.839 & 1.3280 $\pm$ 1.779 & 1.9722 $\pm$ 0.321 \\
& MDGen & 50.1211 $\pm$ 10.584 & 169.6357 $\pm$ 30.496 & 6.0388 $\pm$ 3.477 & 4.2805 $\pm$ 0.244 \\
& \Block[fill=gray!15]{1-}{}\textbf{ProtSCAPE (DDPM)} & 0.6950 $\pm$ 1.286 & \textbf{11.3824 $\pm$ 16.178} & \textbf{0.0000 $\pm$ 0.000} & \textbf{1.3773 $\pm$ 0.640} \\
& \Block[fill=gray!15]{1-}{}\textbf{ProtSCAPE (Flow Matching)} & 1.2103 $\pm$ 1.777 & 18.1489 $\pm$ 21.032 & \textbf{0.0000 $\pm$ 0.000} & 1.6228 $\pm$ 0.688 \\
\bottomrule

\end{NiceTabular}
}
\endgroup
\end{small}
\end{table*}

\section{Ablation Studies}
\label{app:abl}

We further perform ablation studies to evaluate the contribution of individual architectural components in both ensemble generation and minimum energy path estimation. As shown in Tables~\ref{tab:ensemble_ablation} and~\ref{tab:ablation_results}, removing wavelet operators, atomic number features, amino acid identity features, or edge features consistently degrades structural quality across proteins. In particular, the full ProtSCAPE model achieves substantially lower Ramachandran outlier rates, clash scores, and MolProbity scores relative to all ablated variants, highlighting the importance of multiscale wavelet aggregation together with chemically informative node and edge representations. Similar trends are observed for minimum energy path generation, where the complete model produces more physically realistic interpolation trajectories while maintaining competitive $\Delta$ RMSD agreement.

\begin{table*}[!htbp]
\caption{\textbf{Ablation study on ensemble generation quality across protein datasets.}
Structural validity is measured via Ramachandran outliers, steric clash score, rotamer outliers, and MolProbity score (mean $\pm$ std). Best result in each column within a protein block is bolded; a column in which every method ties is left unbolded.}
\label{tab:ensemble_ablation}
\centering
\setlength{\tabcolsep}{4pt}
\begin{small}
\begingroup
\renewcommand{\arraystretch}{1.2}
\adjustbox{max width=0.872\textwidth}{
\begin{NiceTabular}{llcccc}
\toprule
\textbf{Protein} & \textbf{Ablation}
& Rama. Outliers$\downarrow$
& Clash Score$\downarrow$
& Rotamer Outliers$\downarrow$
& MolProbity Score$\downarrow$ \\
\midrule

\multirow{5}{*}{\makecell{6H86\\(Chain A)}}
& GCN (no wavelets)                      & 34.73 $\pm$ 4.52 & 1027.99 $\pm$ 135.54 & 0.00 $\pm$ 0.00 & 4.47 $\pm$ 0.06  \\
& w/o atomic number                     & 29.71 $\pm$ 3.50 & 681.70 $\pm$ 121.03 & 0.00 $\pm$ 0.00 & 4.26 $\pm$ 0.08  \\
& w/o amino acid index \& atomic number & 39.22 $\pm$ 4.61 & 1039.55 $\pm$ 103.95 & 0.00 $\pm$ 0.00 & 4.50 $\pm$ 0.06  \\
& w/o edge features                    & 46.67 $\pm$ 4.66 & 1280.03 $\pm$ 98.54 & 0.00 $\pm$ 0.00 & 4.62 $\pm$ 0.04  \\
& \Block[fill=gray!15]{1-}{}\textbf{ProtSCAPE} & \textbf{0.30 $\pm$ 0.63} & \textbf{16.27 $\pm$ 14.89} & 0.00 $\pm$ 0.00 & \textbf{1.55 $\pm$ 0.52}  \\
\midrule

\multirow{5}{*}{\makecell{6P5H\\(Chain A)}}
& GCN (no wavelets)                      & 61.41 $\pm$ 4.46 & 184.87 $\pm$ 51.04 & 0.00 $\pm$ 0.00 & 3.81 $\pm$ 0.11  \\
& w/o atomic number                     & 28.80 $\pm$ 5.62 & 701.04 $\pm$ 46.73 & 0.00 $\pm$ 0.00 & 4.26 $\pm$ 0.06  \\
& w/o amino acid index \& atomic number & 20.71 $\pm$ 3.50 & 580.80 $\pm$ 54.07 & 0.00 $\pm$ 0.00 & 4.16 $\pm$ 0.05  \\
& w/o edge features                    & 32.23 $\pm$ 4.29 & 812.61 $\pm$ 41.14 & 0.00 $\pm$ 0.00 & 4.36 $\pm$ 0.03  \\
& \Block[fill=gray!15]{1-}{}\textbf{ProtSCAPE} & \textbf{0.76 $\pm$ 0.90} & \textbf{14.35 $\pm$ 13.35} & 0.00 $\pm$ 0.00 & \textbf{1.66 $\pm$ 0.55}  \\
\midrule

\multirow{5}{*}{\makecell{6JV8\\(Chain A)}}
& GCN (no wavelets)                      & 48.08 $\pm$ 6.79 & 468.86 $\pm$ 88.58 & 0.00 $\pm$ 0.00 & 4.19 $\pm$ 0.08  \\
& w/o atomic number                     & 26.09 $\pm$ 2.67 & 679.53 $\pm$ 27.33 & 0.00 $\pm$ 0.00 & 4.19 $\pm$ 0.02  \\
& w/o amino acid index \& atomic number & 24.32 $\pm$ 5.90 & 512.21 $\pm$ 81.81 & 0.00 $\pm$ 0.00 & 4.10 $\pm$ 0.10  \\
& w/o edge features                    & 17.18 $\pm$ 0.62 & 361.59 $\pm$ 2.80 & 0.00 $\pm$ 0.00 & 3.93 $\pm$ 0.01  \\
& \Block[fill=gray!15]{1-}{}\textbf{ProtSCAPE} & \textbf{1.01 $\pm$ 1.32} & \textbf{14.07 $\pm$ 15.99} & 0.00 $\pm$ 0.00 & \textbf{1.57 $\pm$ 0.58}  \\
\midrule

\multirow{5}{*}{\makecell{7LP1\\(Chain A)}}
& GCN (no wavelets)                      & 42.13 $\pm$ 9.91 & 853.19 $\pm$ 130.89 & 0.00 $\pm$ 0.00 & 4.40 $\pm$ 0.07  \\
& w/o atomic number                     & 11.68 $\pm$ 2.71 & 132.67 $\pm$ 45.52 & 0.00 $\pm$ 0.00 & 3.33 $\pm$ 0.21  \\
& w/o amino acid index \& atomic number & 7.03 $\pm$ 2.34 & 187.74 $\pm$ 45.54 & 0.00 $\pm$ 0.00 & 3.40 $\pm$ 0.15  \\
& w/o edge features                    & 6.57 $\pm$ 1.45 & 259.81 $\pm$ 34.68 & 0.00 $\pm$ 0.00 & 3.67 $\pm$ 0.06  \\
& \Block[fill=gray!15]{1-}{}\textbf{ProtSCAPE} & \textbf{0.70 $\pm$ 1.25} & \textbf{20.03 $\pm$ 22.13} & 0.00 $\pm$ 0.00 & \textbf{1.82 $\pm$ 0.60}  \\
\midrule

\multirow{5}{*}{\makecell{7JFL\\(Chain C)}}
& GCN (no wavelets)                      & 45.18 $\pm$ 5.95 & 892.86 $\pm$ 183.12 & 0.00 $\pm$ 0.00 & 4.43 $\pm$ 0.10  \\
& w/o atomic number                     & 15.36 $\pm$ 6.44 & 402.45 $\pm$ 93.30 & 0.00 $\pm$ 0.00 & 3.93 $\pm$ 0.09  \\
& w/o amino acid index \& atomic number & 13.00 $\pm$ 5.17 & 372.10 $\pm$ 54.59 & 0.00 $\pm$ 0.00 & 3.86 $\pm$ 0.08  \\
& w/o edge features                    & 14.02 $\pm$ 4.79 & 418.61 $\pm$ 72.14 & 0.00 $\pm$ 0.00 & 3.86 $\pm$ 0.13  \\
& \Block[fill=gray!15]{1-}{}\textbf{ProtSCAPE} & \textbf{0.42 $\pm$ 1.08} & \textbf{10.41 $\pm$ 11.83} & 0.00 $\pm$ 0.00 & \textbf{1.39 $\pm$ 0.52}  \\
\bottomrule
\end{NiceTabular}
}
\endgroup
\end{small}
\end{table*}

\begin{table*}[!htbp]
\caption{\textbf{Ablation study on ProtSCAPE components for Minimum Energy Path.} Structural validity is measured via Ramachandran outliers, steric clash score, rotamer outliers, and MolProbity score. We also report $\Delta$RMSD (in \AA; lower is better). Values are reported as mean $\pm$ std. Best result in each column within a protein block is bolded; a column in which every method ties is left unbolded.}
\label{tab:ablation_results}
\centering
\setlength{\tabcolsep}{4pt}
\begin{small}
\begingroup
\renewcommand{\arraystretch}{1.2}
\adjustbox{max width=\textwidth}{
\begin{NiceTabular}{llccccc}
\toprule
\textbf{Protein} & \textbf{Ablation}
& Rama. Outliers (\%)$\downarrow$
& Clash Score$\downarrow$
& Rotamer Outliers$\downarrow$
& MolProbity Score$\downarrow$
& $\Delta$RMSD (\AA)$\downarrow$ \\
\midrule

\multirow{5}{*}{\makecell{6H86\\(Chain A)}}
& GCN (no wavelets)                         & 31.40 $\pm$ 6.13 & 866.52 $\pm$ 129.53 & 0.00 $\pm$ 0.00 & 4.37 $\pm$ 0.08 & \textbf{0.12 $\pm$ 0.02}  \\
& w/o atomic number                         & 39.14 $\pm$ 3.86 & 880.87 $\pm$ 132.44 & 0.00 $\pm$ 0.00 & 4.41 $\pm$ 0.08 & 0.13 $\pm$ 0.02  \\
& w/o amino acid index \& atomic number    & 31.19 $\pm$ 4.52 & 893.18 $\pm$ 136.39 & 0.00 $\pm$ 0.00 & 4.38 $\pm$ 0.08 & 0.14 $\pm$ 0.19  \\
& w/o edge features                        & 45.70 $\pm$ 3.76 & 1134.41 $\pm$ 99.55 & 0.00 $\pm$ 0.00 & 4.57 $\pm$ 0.04 & 0.15 $\pm$ 0.19  \\
& \Block[fill=gray!15]{1-}{}\textbf{ProtSCAPE} & \textbf{1.70 $\pm$ 1.65} & \textbf{99.27 $\pm$ 99.80} & 0.00 $\pm$ 0.00 & \textbf{2.44 $\pm$ 0.82} & 0.20 $\pm$ 0.05  \\
\midrule

\multirow{5}{*}{\makecell{6P5H\\(Chain A)}}
& GCN (no wavelets)                         & 57.54 $\pm$ 4.54 & 160.45 $\pm$ 68.59 & 0.00 $\pm$ 0.00 & 3.72 $\pm$ 0.20 & \textbf{0.06 $\pm$ 0.02}  \\
& w/o atomic number                         & 22.49 $\pm$ 4.78 & 584.96 $\pm$ 76.12 & 0.00 $\pm$ 0.00 & 4.13 $\pm$ 0.09 & 0.08 $\pm$ 0.01  \\
& w/o amino acid index \& atomic number    & 18.05 $\pm$ 3.79 & 506.74 $\pm$ 77.88 & 0.00 $\pm$ 0.00 & 4.07 $\pm$ 0.09 & \textbf{0.06 $\pm$ 0.02}  \\
& w/o edge features                        & 29.63 $\pm$ 2.46 & 748.06 $\pm$ 34.22 & 0.00 $\pm$ 0.00 & 4.30 $\pm$ 0.02 & 0.09 $\pm$ 0.01  \\
& \Block[fill=gray!15]{1-}{}\textbf{ProtSCAPE} & \textbf{3.38 $\pm$ 2.16} & \textbf{124.37 $\pm$ 108.04} & 0.00 $\pm$ 0.00 & \textbf{2.73 $\pm$ 0.67} & 0.11 $\pm$ 0.03  \\
\midrule

\multirow{5}{*}{\makecell{6JV8\\(Chain A)}}
& GCN (no wavelets)                         & 49.63 $\pm$ 3.59 & 368.69 $\pm$ 96.66 & 0.00 $\pm$ 0.00 & 4.07 $\pm$ 0.10 & 0.07 $\pm$ 0.02  \\
& w/o atomic number                         & 28.36 $\pm$ 2.11 & 697.40 $\pm$ 24.50 & 0.00 $\pm$ 0.00 & 4.21 $\pm$ 0.02 & 0.04 $\pm$ 0.06  \\
& w/o amino acid index \& atomic number    & 27.43 $\pm$ 4.52 & 474.83 $\pm$ 40.20 & 0.00 $\pm$ 0.00 & 4.09 $\pm$ 0.05 & 0.05 $\pm$ 0.01  \\
& w/o edge features                        & 16.02 $\pm$ 1.86 & 402.06 $\pm$ 31.92 & 0.00 $\pm$ 0.00 & 3.96 $\pm$ 0.03 & \textbf{0.03 $\pm$ 0.01}  \\
& \Block[fill=gray!15]{1-}{}\textbf{ProtSCAPE} & \textbf{5.57 $\pm$ 2.85} & \textbf{154.15 $\pm$ 133.61} & 0.00 $\pm$ 0.00 & \textbf{2.81 $\pm$ 0.97} & 0.11 $\pm$ 0.04  \\
\midrule

\multirow{5}{*}{\makecell{7LP1\\(Chain A)}}
& GCN (no wavelets)                         & 42.36 $\pm$ 5.15 & 613.13 $\pm$ 202.36 & 0.00 $\pm$ 0.00 & 4.26 $\pm$ 0.13 & 0.09 $\pm$ 0.02  \\
& w/o atomic number                         & 14.26 $\pm$ 2.45 & \textbf{98.88 $\pm$ 55.73} & 0.00 $\pm$ 0.00 & \textbf{3.14 $\pm$ 0.27} & 0.04 $\pm$ 0.03  \\
& w/o amino acid index \& atomic number    & 6.05 $\pm$ 2.63 & 143.22 $\pm$ 49.38 & 0.00 $\pm$ 0.00 & 3.29 $\pm$ 0.17 & 0.04 $\pm$ 0.02  \\
& w/o edge features                        & 6.99 $\pm$ 1.34 & 325.77 $\pm$ 42.73 & 0.00 $\pm$ 0.00 & 3.77 $\pm$ 0.07 & \textbf{0.02 $\pm$ 0.02}  \\
& \Block[fill=gray!15]{1-}{}\textbf{ProtSCAPE} & \textbf{2.22 $\pm$ 1.64} & 213.27 $\pm$ 126.70 & 0.00 $\pm$ 0.00 & 3.24 $\pm$ 0.36 & 0.08 $\pm$ 0.05  \\
\midrule

\multirow{5}{*}{\makecell{7JFL\\(Chain C)}}
& GCN (no wavelets)                         & 50.18 $\pm$ 7.56 & 468.64 $\pm$ 270.61 & 0.00 $\pm$ 0.00 & 4.09 $\pm$ 0.27 & \textbf{0.08 $\pm$ 0.03}  \\
& w/o atomic number                         & 18.88 $\pm$ 7.31 & 330.93 $\pm$ 150.49 & 0.00 $\pm$ 0.00 & 3.82 $\pm$ 0.21 & 0.10 $\pm$ 0.18  \\
& w/o amino acid index \& atomic number    & 15.36 $\pm$ 5.60 & 285.34 $\pm$ 85.89 & 0.00 $\pm$ 0.00 & 3.72 $\pm$ 0.16 & 0.11 $\pm$ 0.03  \\
& w/o edge features                        & 15.38 $\pm$ 5.77 & 395.81 $\pm$ 61.10 & 0.00 $\pm$ 0.00 & 3.87 $\pm$ 0.12 & 0.10 $\pm$ 0.03  \\
& \Block[fill=gray!15]{1-}{}\textbf{ProtSCAPE} & \textbf{2.51 $\pm$ 3.20} & \textbf{183.63 $\pm$ 149.23} & 0.00 $\pm$ 0.00 & \textbf{2.92 $\pm$ 0.76} & 0.10 $\pm$ 0.06  \\
\bottomrule
\end{NiceTabular}
}
\endgroup
\end{small}
\end{table*}

\section{Energy Descent}
\label{app:energy_descent}

Table~\ref{tab:noiseless-rmsd-energy-success} reports path-level evaluation metrics for noiseless energy-descent trajectories across proteins and ablation settings. Overall, the full ProtSCAPE model achieves the strongest and most consistent energy descent behavior, obtaining the highest aggregated energy decrease success and monotonic descent rates while maintaining competitive pairwise RMSD agreement to reference trajectories.

\input{energy_path_metrics.tex}

\section{NEB Energy Profile}

Here we look at the energy profiles along the NEB path, we isolated the pure torsional energy components across the trajectory frames (as seen in Figure~\ref{fig:neb_energy}). The resulting profiles map a clear rotational energy barrier for each system, starting and ending at a stable, low energy state. The energy peaks observed near the midpoints (Frames 40-60) correspond to the structurally hindered transition states.

\begin{figure*}[!t]
\centering
\includegraphics[width=1.00\textwidth]{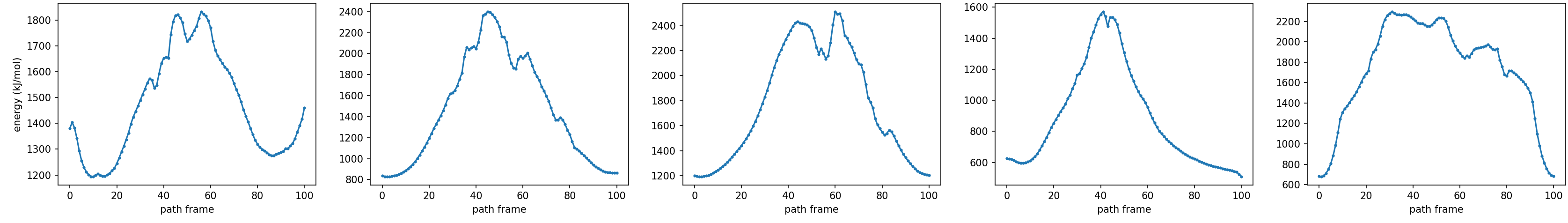}
\caption{Isolated conformational potential energy profiles across the path frames for the five protein systems. The energy profiles reveal the pure rotational activation barriers, characterized by the steric transition state regions peaking near the path midpoints (Frames 40–60)}.
\label{fig:neb_energy}
\end{figure*}

%% file: energy_path_metrics.tex
\begin{table*}[!htbp]
\caption{\textbf{Path-level metrics for noiseless energy-descent paths.} Aggregated reports use all 50 paths per method (10 per protein); per-protein blocks use the 10 paths for that protein. Pairwise RMSD difference: structural spread match to reference (lower is better). Energy decrease success: percentage of paths with $E_{\text{end}} < E_{\text{start}}$. Consistent descent: percentage of paths with monotonic descent score $> 0.5$. Best in each block bolded.}
\label{tab:noiseless-rmsd-energy-success}
\centering
\setlength{\tabcolsep}{5pt}
\begin{small}
\begingroup
\renewcommand{\arraystretch}{1.2}
\adjustbox{max width=\textwidth}{
\begin{NiceTabular}{llccc}
\toprule
\textbf{Group} & \textbf{Method}
& Pairwise RMSD Diff.$\downarrow$
& Energy Decrease Success$\uparrow$
& Consistent Descent$\uparrow$ \\
\midrule

\multirow{5}{*}{\makecell{6H86\\(Chain A)}} & \Block[fill=gray!15]{1-}{}\textbf{ProtSCAPE} & 0.6016 $\pm$ 0.2912 & 80.0\% & 60.0\% \\
 & GCN & \textbf{0.3290 $\pm$ 0.0494} & \textbf{100.0\%} & \textbf{80.0\%} \\
 & w/o atomic number & 0.8405 $\pm$ 0.0684 & 80.0\% & 30.0\% \\
 & w/o atomic number \& amino acid index & 0.8572 $\pm$ 0.0452 & 50.0\% & 20.0\% \\
 & w/o edge features & 0.7145 $\pm$ 0.1357 & 50.0\% & 0.0\% \\
\midrule

\multirow{5}{*}{\makecell{6P5H\\(Chain A)}} & \Block[fill=gray!15]{1-}{}\textbf{ProtSCAPE} & \textbf{0.6669 $\pm$ 0.0655} & \textbf{90.0\%} & \textbf{90.0\%} \\
 & GCN & 0.7376 $\pm$ 0.0975 & 10.0\% & 0.0\% \\
 & w/o atomic number & 0.8292 $\pm$ 0.0376 & 60.0\% & 0.0\% \\
 & w/o atomic number \& amino acid index & 0.8891 $\pm$ 0.0527 & 70.0\% & 0.0\% \\
 & w/o edge features & 0.8536 $\pm$ 0.0472 & 30.0\% & 0.0\% \\
\midrule

\multirow{5}{*}{\makecell{6JV8\\(Chain A)}} & \Block[fill=gray!15]{1-}{}\textbf{ProtSCAPE} & \textbf{0.5843 $\pm$ 0.0824} & \textbf{100.0\%} & \textbf{100.0\%} \\
 & GCN & 0.7501 $\pm$ 0.0662 & 60.0\% & 0.0\% \\
 & w/o atomic number & 0.9354 $\pm$ 0.0182 & 60.0\% & 0.0\% \\
 & w/o atomic number \& amino acid index & 0.8160 $\pm$ 0.0526 & 90.0\% & 60.0\% \\
 & w/o edge features & 0.9974 $\pm$ 0.0003 & \textbf{100.0\%} & 0.0\% \\
\midrule

\multirow{5}{*}{\makecell{7LP1\\(Chain A)}} & \Block[fill=gray!15]{1-}{}\textbf{ProtSCAPE} & 0.7258 $\pm$ 0.0351 & \textbf{100.0\%} & \textbf{100.0\%} \\
 & GCN & \textbf{0.6791 $\pm$ 0.1212} & 90.0\% & 50.0\% \\
 & w/o atomic number & 0.8104 $\pm$ 0.0170 & \textbf{100.0\%} & 60.0\% \\
 & w/o atomic number \& amino acid index & 0.7049 $\pm$ 0.0814 & 70.0\% & 0.0\% \\
 & w/o edge features & 0.8756 $\pm$ 0.0100 & \textbf{100.0\%} & \textbf{100.0\%} \\
\midrule

\multirow{5}{*}{\makecell{7JFL\\(Chain C)}} & \Block[fill=gray!15]{1-}{}\textbf{ProtSCAPE} & 0.7741 $\pm$ 0.0340 & \textbf{100.0\%} & \textbf{100.0\%} \\
 & GCN & 0.7672 $\pm$ 0.0720 & 70.0\% & 20.0\% \\
 & w/o atomic number & 0.7463 $\pm$ 0.0110 & \textbf{100.0\%} & \textbf{100.0\%} \\
 & w/o atomic number \& amino acid index & \textbf{0.7404 $\pm$ 0.0181} & \textbf{100.0\%} & \textbf{100.0\%} \\
 & w/o edge features & 0.7798 $\pm$ 0.0324 & \textbf{100.0\%} & \textbf{100.0\%} \\

\midrule

\multirow{5}{*}{Aggregated} & \Block[fill=gray!15]{1-}{}\textbf{ProtSCAPE} & 0.6705 $\pm$ 0.1528 & \textbf{94.0\%} & \textbf{90.0\%} \\
 & GCN & \textbf{0.6526 $\pm$ 0.1851} & 66.0\% & 30.0\% \\
 & w/o atomic number & 0.8324 $\pm$ 0.0710 & 80.0\% & 38.0\% \\
 & w/o atomic number \& amino acid index & 0.8015 $\pm$ 0.0871 & 76.0\% & 36.0\% \\
 & w/o edge features & 0.8442 $\pm$ 0.1153 & 76.0\% & 40.0\% \\
 
\bottomrule
\end{NiceTabular}
}
\endgroup
\end{small}
\end{table*}